\documentclass{article}

    \PassOptionsToPackage{numbers, compress}{natbib}
\usepackage[preprint]{neurips_2025}

\usepackage[utf8]{inputenc} 
\usepackage[T1]{fontenc}    
\usepackage{hyperref}       
\usepackage{url}            
\usepackage{booktabs}       
\usepackage{amsfonts}       
\usepackage{nicefrac}       
\usepackage{microtype}      
\usepackage{xcolor}         

\usepackage{graphicx}             
\usepackage{subfigure}            
\usepackage{array}                
\usepackage{multirow}             

\usepackage{amsmath}              
\usepackage{amssymb}              
\usepackage{mathtools}            

\usepackage{nicefrac}             

\usepackage{hyperref}             
\usepackage{url}                  

\usepackage[capitalize,noabbrev]{cleveref}

\usepackage{algorithm}
\usepackage{algpseudocode}

\usepackage{caption}

\usepackage{bm}

\usepackage{amsthm}

\newtheorem{theorem}{Theorem}

\theoremstyle{remark}
\newtheorem{remark}{Remark}
\newtheorem{corollary}[theorem]{Corollary}
\newtheorem{proposition}{Proposition}
\newtheorem{assumption}{Assumption}

\newcommand{\gurobi}{\textsl{Gurobi}}

\newcommand{\scip}{\textsl{SCIP}}

\newcommand{\ipopt}{\textsl{Ipopt}}
\newcommand{\baron}{\textsl{BARON}}
\newcommand{\antigone}{\textsl{ANTIGONE}}
\newcommand{\couenne}{\textsl{COUENNE}}

\title{Learning to Optimize for Mixed-Integer Non-linear Programming with Feasibility Guarantees}

\author{%
  Bo Tang \\
  Department of Mechanical and Industrial Engineering \\
  University of Toronto \\
  Toronto, ON M5S 1A1, Canada \\
  \texttt{bo.tang@mail.utoronto.ca} \\
  \And
  Elias B.~Khalil \\
  Department of Mechanical and Industrial Engineering \\
  University of Toronto \\
  Toronto, ON M5S 1A1, Canada \\
  \texttt{elias.khalil@mie.utoronto.ca} \\
  \And
  Ján Drgoňa \\
  Department of Civil and Systems Engineering \\
  The Ralph O'Connor Sustainable Energy Institute (ROSEI) \\
  Data Science and AI Institute (DSAI) \\
  Johns Hopkins University \\
  Baltimore, MD 21218, USA \\
  \texttt{jdrgona1@jh.edu} \\
}

\begin{document}

\maketitle

\begin{abstract}
Mixed-integer nonlinear programs (MINLPs) arise in domains such as energy systems, process engineering, and transportation, and are notoriously difficult to solve at scale due to the interplay of discrete decisions and nonlinear constraints. In many practical settings, these problems appear in parametric form, where objectives and constraints depend on instance-specific parameters, creating the need for fast and reliable solutions across related instances. While learning-to-optimize (L2O) methods have shown strong performance in continuous optimization, extending them to MINLPs requires enforcing both feasibility and integrality within a data-driven framework. We propose an L2O approach tailored to parametric MINLPs that generates instance-specific solutions using integer correction layers to enforce integrality and a gradient-based projection to ensure feasibility of the inequality constraints. Theoretically, we provide asymptotic and non-asymptotic convergence guarantees of the projection step. Empirically, the framework scales to MINLPs with tens of thousands of variables and produces feasible high-quality solutions within milliseconds, often outperforming traditional solvers and heuristic baselines in repeated-solve settings.
\end{abstract}

\section{Introduction}

Mixed-integer optimization arises in a broad range of real-world applications, including pricing~\cite{KLEINERT2021100007}, battery dispatch~\citep{nazir2021guaranteeing}, transportation~\citep{Schouwenaars2001}, and optimal control~\citep{Marcucci2021}. While many such problems can be represented as mixed-integer linear programs (MILPs), a substantial class of high-impact applications additionally involves nonlinear physical, engineering, or economic relationships. These problems combine discrete and nonlinear components and are naturally modeled as mixed-integer nonlinear programs (MINLPs). Examples include alternating current optimal power flow~\citep{frank2016introduction}, energy generation scheduling~\citep{frangioni2006solving}, and natural gas network operations~\citep{hahn2017mixed}, as well as production planning~\citep{sridhar2014models}, process design~\citep{biegler1997systematic}, and portfolio optimization~\citep{mansini2015linear}. Although MILP solvers are highly mature and nonlinear relationships are often approximated using piecewise-linear models, these approximations may not fully capture the underlying physics or economics and can lead to nontrivial modeling errors. This motivates the need for scalable algorithms that operate directly on MINLPs rather than rely solely on MILP reformulations.

Unlike MILPs, which benefit from decades of advances in exact algorithms~\citep{land2010automatic} and heuristics~\citep{crama2005local, johnson1997traveling}, MINLPs remain considerably more challenging due to the interaction between discrete variables and nonlinear objective and constraint functions. Classical MINLP solvers rely on decomposition and tree-search strategies and, even in the convex case, must repeatedly solve large nonlinear relaxations. Foundational approaches such as outer approximation~\citep{fletcher1994solving}, spatial branch-and-bound~\citep{belotti2009branching}, and related decomposition schemes~\citep{nowak2005relaxation} rely on alternating between MILP master problems and NLP subproblems or on iteratively refining cutting planes. These procedures often incur substantial computational cost as the problem size grows. For nonconvex MINLPs, global solvers such as \baron{}~\citep{sahinidis1996baron}, \antigone{}~\citep{misener2014antigone}, and \couenne{}~\citep{belotti2009branching} still face worst-case exponential search trees and expensive nonlinear solves. Considerable solver effort is frequently devoted to finding feasible solutions at all, motivating primal heuristics such as feasibility pump, diving, and local branching~\citep{berthold2015heuristic}. Although helpful in improving incumbents, these heuristics do not provide global guarantees and become ineffective in highly nonconvex landscapes. Thus, scalability remains a challenge for general-purpose MINLP solvers, particularly in large or time-critical applications.

In many real-world applications, the same optimization model must be solved repeatedly under changing inputs: For example, when delivery routes are updated, production schedules are revised, or portfolio allocations are rebalanced. Although each instance differs in its parameters, the underlying combinatorial or algebraic structure remains fixed. Learning-to-Optimize (L2O) leverages this repeated structure by training models on families of related instances to replace or accelerate components of conventional solvers. Existing work has demonstrated substantial speedups for linear, integer, and quadratic programs~\citep{park2023self, khalil2016learning, donti2021dc3}. However, extending L2O to more general settings such as MINLP remains largely unexplored, motivating the present work.

Recent progress in machine learning for optimization has largely followed two complementary directions. The first augments classical algorithms by learning components within the solver pipeline (branching, cutting, heuristics, or other algorithmic parameters) while preserving the structure and guarantees of exact methods~\citep{khalil2016learning, li2024pdhg}. The second predicts solutions directly through end-to-end architectures or other solver-free mappings~\citep{donti2021dc3, kool2018attention}. Although these approaches have demonstrated substantial speedups, extending them to mixed-integer nonlinear programs is difficult. Solver-augmented methods still depend on expensive nonlinear relaxations, whereas solver-free models lack built-in mechanisms to enforce combinatorial or nonlinear feasibility.

\paragraph{Contributions.}
This work makes the following contributions: 
(i) we propose the first general-purpose \emph{solver-free} L2O framework for parametric MINLPs that generates integer solutions directly, without relying on continuous relaxations or solver-augmented components; 
(ii) we introduce two differentiable integer-correction layers that propagate gradient information through discrete operations, enabling neural networks to construct and iteratively refine integer solutions; 
(iii) we develop a lightweight feasibility-projection heuristic that leverages gradient signals from integer correction layers to satisfy mixed-integer inequality constraints  without invoking external solvers; 
(iv) we provide asymptotic and non-asymptotic theoretical guarantees establishing conditions under which the projection step converges and recovers feasible solutions; and 
(v) we demonstrate scalable empirical performance on challenging benchmarks, achieving high-quality solutions with substantial speedups on repeated parametric instances.

\section{Related Work}

L2O seeks to improve optimization efficiency or scalability using data-driven models. Existing work can be broadly divided into \emph{solver-augmented} methods that guide or modify classical algorithms, and \emph{solver-free} methods that attempt to predict solutions directly. While these approaches have shown substantial progress for continuous, extending them to general MINLPs remains challenging.

\subsection{Solver-Augmented Methods}

\paragraph{Algorithm Unrolling.}
Algorithm unrolling has emerged as an influential paradigm for integrating classical continuous optimization with data-driven learning. Early work, such as LISTA~\citep{gregor2010learning} demonstrated that iterative soft-thresholding for sparse coding can be reformulated as a neural network whose step sizes and linear transforms are learned from data. This idea has since been extended to a variety of continuous optimization algorithms. For example, PDHG-Net~\citep{li2024pdhg, yang2024efficient} unrolls primal–dual hybrid gradient iterations into a trainable architecture, enabling adaptive update rules while retaining the structure of the original method.

\paragraph{Learning to Guide.}
There is also a substantial literature that leverages machine learning to guide decision-making within classical discrete solvers. This work retains the branch-and-cut framework while learning improved branching~\citep{khalil2016learning, alvarez2017machine, GCNN, zarpellon2021parameterizing}, node selection~\citep{he2014learning, mattick2023reinforcement}, and cut selection~\citep{Deza_2023, wang2023learning}. Related ideas have been used to enhance primal heuristics in routing~\citep{xin2021neurolkh}. These approaches reduce search effort by adapting solver behavior to instance structure, but feasibility and optimality are guaranteed by the solver itself.

\paragraph{Predict and Search.}
A complementary research stream uses machine learning to predict partial solution information, which is then used to restrict or initialize the subsequent search. Representative approaches include variable fixing and data-driven neighborhood design~\citep{ding2020accelerating, song2020general}, as well as learned partial assignments such as those used in Neural Diving~\citep{nair2020solving}. More recent work employs GNN-based predictors to repeatedly propose promising neighborhoods that are refined by a solver within a large-neighborhood search framework~\citep{han2023gnn, huangcontrastive}. In all such methods, machine learning serves to prioritize or narrow the search space, while the solver remains responsible for producing a feasible solution.

\subsection{Solver-Free Methods}

\paragraph{Supervised Learning.}
Supervised prediction learns a mapping from problem inputs to decisions by imitating solutions computed offline. Early neural sequence models such as Pointer Networks~\citep{vinyals2015pointer} and their routing extensions~\citep{joshi2019efficient} demonstrated that neural architectures can approximate sequential solutions end-to-end. Related ideas directly predict decision variables, where predictors approximate feasible solutions in settings such as power flow~\citep{fioretto2020predicting}. Despite their efficiency at inference time, supervised methods require large datasets of optimal or near-optimal solutions~\citep{gleixner2021miplib}, which are expensive to compute and limit scalability to larger problem instances. 

\paragraph{Reinforcement Learning.}
Reinforcement learning constructs solutions through sequential decision-making without relying on supervised labels. Policy-gradient has been used to generate routing solutions~\citep{bello2016neural, kool2018attention}, while graph-based models have learned node-selection strategies for problems such as maximum cut and vertex cover~\citep{khalil2017learning}. However, its performance is highly sensitive to reward design and exploration strategy, and training is often computationally expensive and unstable~\citep{vesselinova2020learning}, which limits its practicality for large or tightly constrained optimization tasks.

\paragraph{Self-Supervised Learning.}
Self-supervised approaches eliminate the need for labeled solutions by minimizing differentiable measures of optimality or feasibility. Representative formulations include primal–dual residual minimization~\citep{park2023self} and Lagrangian-like loss that combine objective value and constraint violation~\citep{donti2021dc3}. More recently, recurrent architectures~\citep{luken2024self} refine feasibility through iterative evaluations of differentiable residuals. While promising for continuous domains, these methods lack feasibility guarantees and do not support the integer-valued outputs required for the MINLP regime studied in this work.

\subsection{Constraint-Aware Learning}

Although L2O has progressed rapidly, enforcing feasibility remains a central challenge. Standard predictive models such as linear regression and neural networks are inherently unconstrained, requiring additional mechanisms to ensure that their outputs respect constraints. Existing approaches broadly fall into hard-constraint and soft-constraint formulations.

\paragraph{Hard Constraint.}
Hard-constraint methods typically integrate feasibility directly into the model architecture or output parameterization. Examples include neural layers that preserve linear operator constraints by construction~\citep{hendriks2020linearly}, as well as reparameterization schemes that map unconstrained predictions into feasible sets to satisfy equalities or inequalities exactly~\citep{frerix2020homogeneous, rangarajan2022expressing}. Differentiable barrier-based layers~\citep{kervadec2022constrained} emulate interior-point updates to enhance constraint adherence. DC3~\citep{donti2021dc3} introduces a differentiable correction mechanism that adjusts infeasible predictions toward the feasible region through gradient-based updates. Although effective for continuous feasibility, these approaches do not extend to enforcing integrality.

\paragraph{Soft Constraint.}
Soft-constraint methods incorporate feasibility through differentiable penalty terms, drawing on classical penalty and augmented Lagrangian techniques~\citep{pathak2015constrained, jia2017constrained}. While they do not provide formal guarantees, such losses can improve approximate feasibility and training stability in high-dimensional settings~\citep{marquez2017imposing}. Penalty-based formulations have been widely used to encode physical or operational constraints in continuous domains~\citep{fioretto2020predicting, pan2020deepopf}. However, unlike hard-constraint mechanisms, penalty-based approaches cannot enforce exact feasibility and introduce additional weighting parameters that are often problem-dependent.

\subsection{Handling Discrete Outputs}

Despite progress in constraint-aware learning, solver-free L2O methods face a fundamental limitation for discrete problems: the inability of standard gradient-based models to represent and learn discrete decision variables. Integer-valued mappings are piecewise constant and nondifferentiable, causing gradients to vanish almost everywhere and preventing effective end-to-end training.

\paragraph{Surrogate Relaxations.}
A common strategy for handling discrete decisions is to replace integer operators with continuous surrogates, such as surrogate-based linearizations~\citep{bonami2022classifier, ferber2023surco} or Gumbel-noise reparameterizations~\citep{jang2016categorical, maddison2016concrete}. While these relaxations enable gradient-based training, they optimize over a continuous superset of the original discrete feasible region. As a result, large integrality gaps between the relaxed and integer problems~\citep{schrijver1998theory, vazirani2001approximation} can lead to substantial misalignment between the learned objective and the true discrete objective, causing divergence from integer optimality.

\paragraph{Straight-Through Estimators for Binary Mappings.}
A complementary technique is the \emph{Straight-Through Estimator} (STE)~\citep{bengio2013estimating}, which enables gradient flow through discrete mappings by replacing the zero derivative of rounding or binarization. In its simplest form, STE approximates the backward Jacobian of a discrete operator by the identity matrix. STE has been widely used in quantized neural networks~\citep{courbariaux2015binaryconnect, hubara2016binarized}. However, STE is agnostic to the structure of the discrete decision space and does not guide the model toward more favorable discrete decisions.

\section{Problem Formulation}

We consider a parametric mixed-integer nonlinear program (pMINLP), where each instance is identified by a parameter vector ${\bm{\xi}} \in \mathbb{R}^{n_{\xi}}$. For a given realization of ${\bm{\xi}}$, the problem takes the form
\begin{equation}
\vspace{-6pt}
\label{eq:minlp}
\min_{{\bm{x}} \in \mathcal{X}} \; f({\bm{x}}, {\bm{\xi}})
\qquad \text{s.t.} \qquad 
{\bm g}({\bm{x}}, {\bm{\xi}}) \le \bm{0},
\vspace{-6pt}
\end{equation}
where the decision space is $\mathcal{X} := \mathbb{R}^{n_r} \times \mathbb{Z}^{n_z},$ and ${\bm{x}} = ({\bm{x}}_r, {\bm{x}}_z)$ contains $n_r$ continuous variables and $n_z$ integer variables. The vector ${\bm g}({\bm{x}},\bm{\xi}) = [g_1({\bm{x}},\bm{\xi}),\dots,g_{n_c}({\bm{x}},\bm{\xi})]^\top $ collects the $n_c$ inequality constraint functions. For each parameter realization ${\bm{\xi}}$, the feasible set is
$
\mathcal{F}({\bm{\xi}})
    := \{\, {\bm{x}} \in \mathcal{X} : {\bm g}({\bm{x}}, {\bm{\xi}}) \le \bm{0} \,\},
$
which contains all mixed-integer decisions satisfying the inequality constraints.

\begin{table}[ht!]
\vspace{-6pt}
\centering
\setlength{\tabcolsep}{6pt}
\renewcommand{\arraystretch}{0.9}
\caption{Notation used throughout the paper.}
\label{tab:notation}
\resizebox{0.98\linewidth}{!}{
\begin{tabular}{ll|ll}
\toprule
\textbf{Symbol} & \textbf{Description} &
\textbf{Symbol} & \textbf{Description} \\
\midrule
$\bm{\xi}$ 
    & Instance (problem) parameters &
$\bm{x} = (\bm{x}_r,\bm{x}_z)$ 
    & Continuous and integer decision variables \\
$\mathcal{X}$ 
    & Decision domain $\mathbb{R}^{n_r} \!\times\! \mathbb{Z}^{n_z}$ &
$\mathcal{F}(\bm{\xi})$
    & Feasible set $\{ \bm{x} \in \mathcal{X} : {\bm g}(\bm{x},\bm{\xi}) \le \bm{0} \}$ \\
$f(\bm{x},\bm{\xi})$
    & Objective function &
${\bm g}(\bm{x},\bm{\xi})$
    & Vector of inequality constraint functions \\
$\bar{\bm{x}}$ 
    & Solution to the continuous relaxation &
$\hat{\bm{x}}$ 
    & Predicted mixed-integer solution \\
$n_r, n_z$
    & Numbers of continuous and integer variables &
$n_c$
    & Number of inequality constraints \\
\bottomrule
\end{tabular}
}
\vspace{-6pt}
\end{table}

\paragraph{Continuous Relaxation}
We also make use of the continuous relaxation of~\Cref{eq:minlp}, obtained by removing the integrality constraints and optimizing over the continuous space $\mathbb{R}^{n_r+n_z}$. Let ${\bar{\bm{x}}}({\bm{\xi}})$ denote an optimal solution of the relaxed problem. Obviously, its solution generally does not satisfy the integrality requirements.  In our framework, ${\bar{\bm{x}}}$ is interpreted as a continuous intermediate representation produced by the neural network and is not required to satisfy the constraints.

\paragraph{Learning Objective}
Since each instance is characterized by its parameter vector ${\bm{\xi}}$, our goal is to learn a prediction model ${\bm{\psi}}_{\Theta} : \mathbb{R}^{n_{\xi}} \to \mathcal{F}({\bm{\xi}})$, with trainable parameters~$\Theta$, that directly produces a mixed-integer prediction ${\hat{\bm{x}}}^i = {\bm{\psi}}_{\Theta}({\bm{\xi}}^i)$ for each instance $i$.

\section{Methodology}

We now present our L2O framework for parametric MINLPs.  \Cref{fig:pipeline} shows that the method consists of two components: 
(i) integer correction layers embedded within the network, and 
(ii) an integer feasibility projection applied at inference. 
For clarity, we omit the instance index $i$ throughout.
\begin{figure}[htbp!]
\vspace{-6pt}
    \centering    \includegraphics[width=0.98\textwidth]{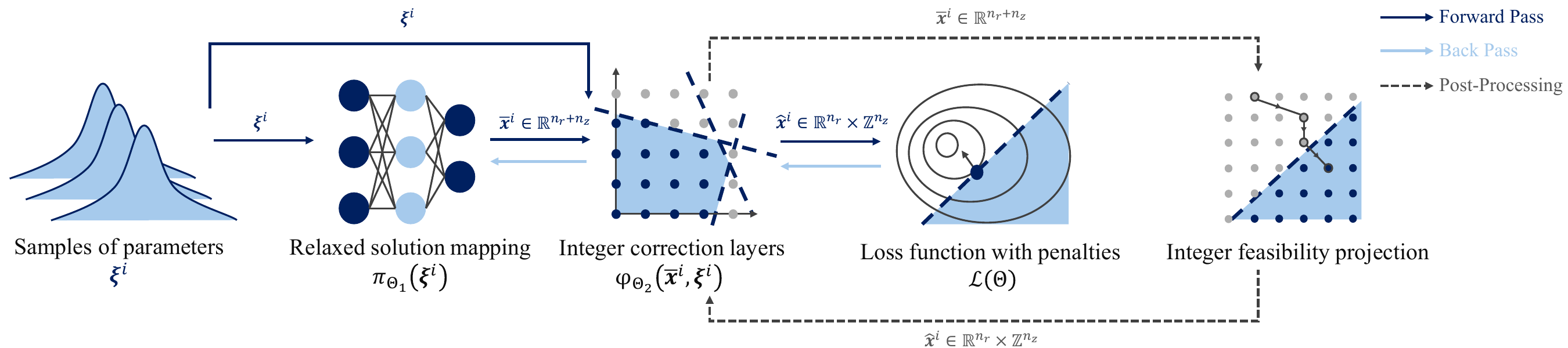}
    \caption{Overview of the proposed solver-free L2O pipeline for parametric MINLPs. Given parameters $\bm{\xi}$, the model generates a relaxed solution, applies integer correction, and refines feasibility through an iterative projection step.}
    \label{fig:pipeline}
\vspace{-6pt}
\end{figure}

\subsection{Learning Objective}

We model the entire architecture, including the relaxed mapping and the integer correction layers, as a single neural network $\bm{\psi}_{\Theta} : \mathbb{R}^{n_{\xi}} \to \mathbb{R}^{n_r} \times \mathbb{Z}^{n_z},$  which maps an instance parameter vector $\bm{\xi}$ to a mixed-integer decision $\hat{\bm{x}}$.
The internal decomposition of $\bm{\psi}_{\Theta}$ into a relaxed solution mapping and an integer correction module
is detailed in~\Cref{subsec:int_corr}.

Our training formulation is self-supervised: we do not assume access to optimal or even feasible solutions as labels and instead rely only on the objective and constraint functions. Given training instances $\bm{\xi}^i\}_{i=1}^m$, we minimize a soft-constrained empirical risk
\begin{equation}
\vspace{-6pt}
\label{eq:loss}
\mathcal{L}(\Theta)
= \frac{1}{m} \sum_{i=1}^m 
\Big[
    f(\hat{\bm{x}}^{i}, \bm{\xi}^{i})
    + \lambda\, \mathcal{V}(\hat{\bm{x}}^{i}, \bm{\xi}^{i})
\Big],
\qquad 
\hat{\bm{x}}^{i} = \bm{\psi}_{\Theta}(\bm{\xi}^{i}),
\vspace{-6pt}
\end{equation}
where $\lambda > 0$ balances optimality and feasibility, and $\mathcal{V}(\bm{x}, \bm{\xi}) := \big\| \bm{g}(\bm{x}, \bm{\xi})_+ \big\|_1$ penalizes violations of the inequality constraints.
Here, $(\cdot)_+$ is the elementwise positive part and $\| \cdot \|_1$ is the $\ell_1$ norm so that it aggregates violations across all constraints. This type of penalty is differentiable almost everywhere and has been widely used in self-supervised constrained learning~\citep{raissi2019physics, donti2021dc3}. The loss $\mathcal{L}(\Theta)$ is minimized using stochastic gradient methods such as \emph{Adam}~\citep{kingma2014adam}.

\subsection{Integer Correction Layers}
\label{subsec:int_corr}

We now formalize the critical component of our framework, the \textit{Integer Correction Layer} summarized in \Cref{algo:l2o}. This module is designed to handle the discrete nature of decision variables in MINLPs by transforming the relaxed outputs of the neural network into mixed-integer solutions. To move beyond fixed rounding heuristics, our correction layers incorporate learnable parameters that adaptively determine the rounding direction of each integer variable based on both the instance context and the relaxed solution. The overall mapping $\psi_{\Theta}:\mathbb{R}^{n_\xi}\mapsto\mathbb{R}^{n_r} \times \mathbb{Z}^{n_z}$ from an instance parameter vector ${\bm{\xi}}$ to a mixed-integer solution ${\hat{\bm{x}}}$ is performed in two stages:
\begin{enumerate}
    \item \textbf{Relaxed Solution Mapping:} The first stage applies a learnable mapping 
    $\bm{\pi}_{\Theta_1} : \mathbb{R}^{n_{\xi}} \mapsto \mathbb{R}^{n_r + n_z}$, parameterized by weights $\Theta_1$, producing a relaxed solution ${\bar{\bm{x}}} = \bm{\pi}_{\Theta_1}({\bm{\xi}})$ without enforcing integrality.
    \item \textbf{Integer Correction:} The second stage refines the relaxed solution ${\bar{\bm{x}}}$ through a correction module $\bm{\varphi}_{\Theta_2}: \mathbb{R}^{n_r + n_z} \times \mathbb{R}^{n_\xi} \mapsto \mathbb{R}^{n_r} \times \mathbb{Z}^{n_z}$, which adaptively determines the rounding direction of each integer variable based on both the instance parameter ${\bm{\xi}}$ and the relaxed solution ${\bar{\bm{x}}}$, producing the final mixed-integer output $\hat{\bm{x}}$.
\end{enumerate}

Conceptually, this design resembles the idea of \textit{Relaxation Enforced Neighborhood Search} (RENS)~\citep{berthold2014rens}, which explores integer-feasible solutions in the neighborhood of a continuous relaxation. In contrast, our approach performs this search implicitly and differentiably via an end-to-end trained neural network.

\begin{figure}[htb!]
\vspace{-15pt}
\centering
\begin{minipage}[t]{0.47\textwidth}
\begin{algorithm}[H]  
\footnotesize
\setlength{\baselineskip}{0.85\baselineskip}
\caption{Integer Correction $\varphi_{\Theta_2}(\bar{\bm x}, {\bm \xi})$}\label{algo:l2o}
\begin{algorithmic}[1]
\State \textbf{Input:} initial relaxed solution $\bar{\bm x}$, parameters ${{\bm \xi}}$, and neural network $\delta_{\Theta_2} (\cdot)$
\State Obtain hidden states ${\bm h} \gets \delta_{\Theta_2} (\bar{\bm x}, {\bm \xi})$
\State Update continuous variables $\hat{{\bm x}}_r \gets \bar{{\bm x}}_r + {\bm h}_r$
\State Round integer variables down $\hat{{\bm x}}_z \gets \lfloor \bar{{\bm x}}_z \rfloor$
\If{using \textit{Rounding Classification} (RC)} 

   \State Obtain values ${\bm v} \gets \text{Gumbel-Sigmoid}({\bm h}_z)$
\ElsIf{using \textit{Learnable Threshold} (LT)} 
  \State Obtain thresholds ${\bm t} \gets \text{Sigmoid} ({\bm h})$
  \State Obtain ${\bm v} \gets \text{Sigmoid}(10 \cdot (\bar{{\bm x}}_z - \hat{{\bm x}}_z -  {\bm r}))$
\EndIf
\State Obtain rounding directions ${\bm b} \gets \mathbb{I} \big({\bm v} > 0.5 \big)$
\State Update integer variables $\hat{{\bm x}}_z \gets \hat{{\bm x}}_z + {\bm b}$
\State \textbf{Output:} a mixed-integer solution $\hat{\bm x}$
\end{algorithmic}
\end{algorithm}
\end{minipage}%
\hfill
\begin{minipage}[t]{0.50\textwidth}
\begin{algorithm}[H]
\footnotesize
\setlength{\baselineskip}{0.85\baselineskip}
\caption{Feasibility Projection $\phi(\bar{\bm x}, \bm \xi)$}\label{algo:proj}
\begin{algorithmic}[1]
\State \textbf{Input:} initial relaxed solution $\bar{\bm x}$,  parameters $\bm \xi$, integer correction layer $\varphi_{\Theta_2} (\cdot)$, and step size $\eta$
\While{True}
  \State Obtain updated integers $\hat{\bm x} \gets \varphi_{\Theta_2}(\bar{\bm x}, \bm \xi)$ (\Cref{algo:l2o})
  \State Compute violations $\mathcal{V}( \hat{\bm x}, \bm \xi) \gets  \| {\bm g}({\bm \hat{x}}, \bm \xi)_+ \|_1$
  \If{$\mathcal{V}( \hat{\bm x}, \bm \xi) = 0$}
    \State Break
  \Else
    \State Compute gradients ${\bm d} \gets \nabla_{\bar{\bm x}} \mathcal{V}(\hat{\bm x}, \bm \xi)$
    \State Update relaxed solution $\bar{\bm x} \gets \bar{\bm x} - \eta {\bm d}$
  \EndIf
\EndWhile
\State \textbf{Output:} a mixed-integer solution $\hat{\bm x}$
\end{algorithmic}
\end{algorithm}
\end{minipage}
\vspace{-12pt}
\end{figure}

The integer correction layers achieve differentiable rounding through the combination of the STE and Sigmoid functions. The neural network $\delta_{\Theta_2}(\bar{\bm{x}}, \bm{\xi})$ outputs hidden representations that determine rounding directions for each integer variable, thereby learning adaptive rounding behaviors conditioned on the instance ${\bm{\xi}}$. During training, the loss in~\Cref{eq:loss} jointly updates the network parameters $\Theta = \Theta_1 \cup \Theta_2$, accounting for both the objective value and constraint violations of the predicted mixed-integer solution $\bm{\hat{x}}$. The overall end-to-end mapping from an instance parameter ${\bm{\xi}}$ to a mixed-integer solution ${\hat{\bm{x}}}$ is expressed as
\begin{equation}
\label{eq:psi_nested}
    \hat{\bm{x}} = \bm{\psi}_{\Theta}(\bm{\xi}) 
    = \bm{\varphi}_{\Theta_2}\big( \bm{\pi}_{\Theta_1}(\bm{\xi}), \bm{\xi} \big),
\end{equation}

We propose two alternative designs for the correction layer ${\bm{\varphi}}_{\Theta_2}$: \textit{Rounding Classification} (RC) and \textit{Learnable Threshold} (LT). Both operate downstream of the relaxed-solution mapping $\bm{\pi}_{\Theta_1}: \mathbb{R}^{n_\xi} \to \mathbb{R}^{n_r+n_z}$ and take as inputs the relaxed output ${\bm{\bar{x}}}=\bm{\pi}_{\Theta_1}({\bm{\xi}})$ together with the instance parameters ${\bm{\xi}}$. A learnable component $\delta_{\Theta_2}(\bar{\bm{x}}, \bm{\xi})$ then produces hidden representations used to decide rounding. Despite this common interface, the mechanisms differ, as illustrated in~\Cref{fig:method_examples}: RC uses a probabilistic classification scheme via logits to choose the rounding direction for each integer variable, whereas LT predicts a continuous threshold vector that controls the rounding decision. Additional implementation details and theoretical analyses of RC and LT, including gradient derivations and Lipschitz smoothness bounds, are provided in Appendix~\ref{app:corr}.

\begin{figure}[htb!]
\vspace{-6pt}
    \centering      
    \includegraphics[width=0.95\textwidth]{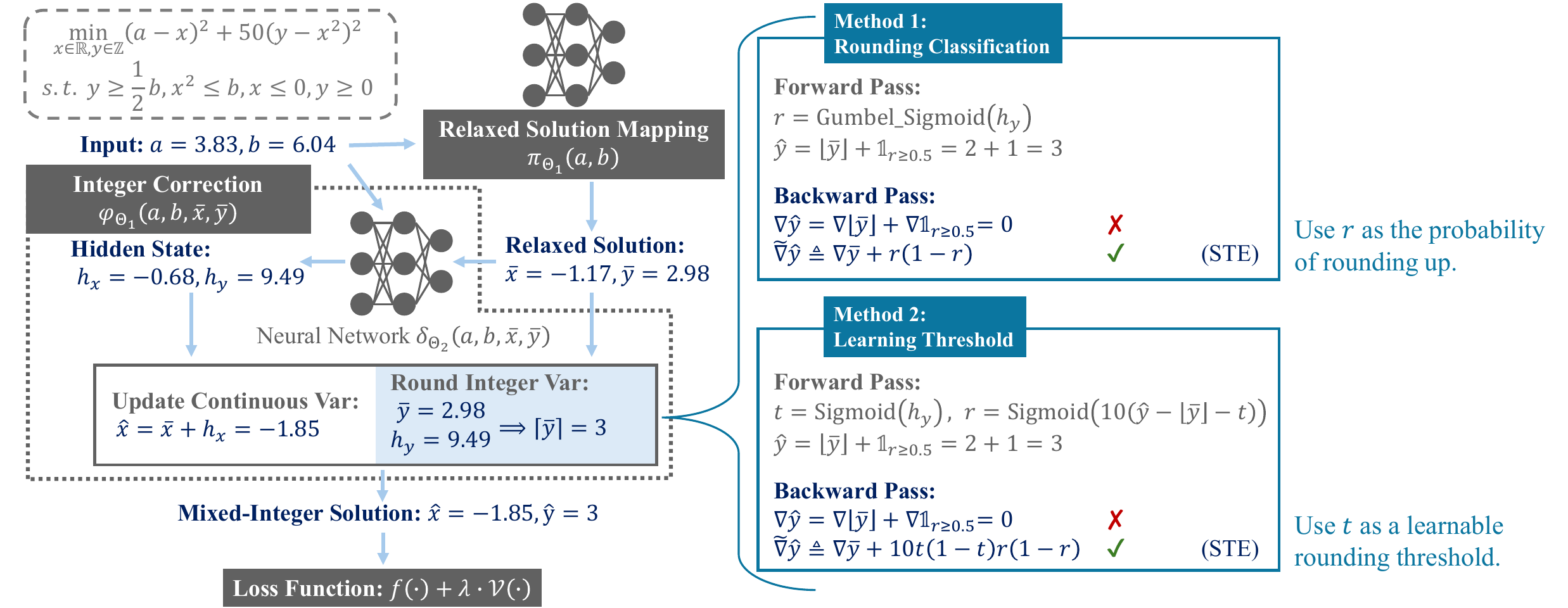}
    \caption{Two integer correction layers: (a) Rounding Classification and (b) Learning Threshold, both using STE for backpropagation.}
    \label{fig:method_examples}
\vspace{-6pt}
\end{figure}

\subsection{Integer Feasibility Projection}

While the proposed integer correction layers effectively enforce integrality, the soft-penalty formulation in~\Cref{eq:loss} cannot always ensure full constraint satisfaction. To address this limitation, we introduce an \textit{integer feasibility projection} procedure at inference time summarized in \Cref{algo:proj}. 

Starting from an initial relaxed solution $\bar{\bm{x}}$ predicted by the network, the projection iteratively refines the solution to reduce constraint violations via gradient-based updates in the continuous space. At each iteration, the relaxed $\bar{\bm{x}}$ is passed through the integer correction layer $\bm{\varphi}_{\Theta_2}(\bar{\bm{x}}, \bm{\xi})$ to produce an integer candidate $\hat{\bm{x}}$. Constraint violations are evaluated on $\hat{\bm{x}}$, and the resulting violation signal is used to update the relaxed $\bar{\bm{x}}$, thereby alternating between continuous feasibility refinement and integer correction. This process continues until convergence to a feasible integer solution.

In \Cref{algo:proj}, line~3 applies the integer correction layer to enforce integrality, while line~9 performs a gradient-based update on the relaxed variable to reduce the constraint violation $\mathcal{V}(\hat{\bm{x}})$, following the principle of~\citet{donti2021dc3}.
Unlike~\citet{donti2021dc3}, who integrate such projections within the training loop, we apply the projection exclusively during inference and use hard rounding in the forward pass to ensure integer-valued outputs. Conceptually, this iterative structure resembles the classical \textit{Feasibility Pump}~\citep{fischetti2005feasibility}, which alternates between rounding and projection to obtain feasible solutions. However, whereas the feasibility pump relies on repeatedly solving constrained subproblems, our integer correction step is learned offline.

While we apply this projection only at inference time in this work, we note that, in principle, the proposed integer feasibility projection could be differentiated through and integrated into end-to-end training. Differentiating through such projections requires maintaining deep computational graphs and handling implicit or higher-order gradient information, which introduces substantial computational and memory overhead and may lead to numerically unstable or low-quality gradient information in practice.  We therefore leave differentiable projection to future work.

\section{Theoretical Guarantees}

This analysis clarifies under what regularity conditions the correction layer ensures feasibility and how fast constraint violations diminish during training. To establish feasibility guarantees, we analyze the convergence behavior of the integer feasibility projection $\phi({\bm{x}}, {\bm{\xi}})$ for a fixed parametric instance ${\bm{\xi}}$. For simplicity of exposition, we omit the dependence on ${\bm{\xi}}$ and denote the integer correction layer by $\varphi({\bm{x}})$, the constraint functions by ${\bm{g}}({\bm{x}}) = [g_1({\bm{x}}), \ldots, g_{n_c}({\bm{x}})]$.
\Cref{thm:relu_exterior_convergence} provides conditions for asymptotic convergence, while \Cref{thm:non_asymptotic_relu} provides non-asymptotic convergence to the approximate feasible set. This result follows from the standard convergence rate of gradient descent on $L$-smooth functions and quantifies how fast the penalty gradient vanishes.

\begin{remark}
The results in this section establish convergence to an \emph{approximately feasible} set.
The analysis does not explicitly quantify the approximation error induced by the surrogate
gradients used for rounding.
\end{remark}

\subsection{Asymptotic Convergence of Integer Feasibility Projection}

\begin{assumption}[Regularity]
\label{assump:regularity}
The constraint functions ${\bm g}(\bm{x},\bm{\xi})$ and the integer-correction mapping $\varphi(\bm{x})$ are continuously differentiable, and their gradients are Lipschitz continuous on the domain of interest.
\end{assumption}

\begin{assumption}[Bounded constraint activity]
\label{assump:bounded_constraints}
The number of violated constraints is uniformly bounded along the trajectory generated by the algorithm.
\end{assumption}

\begin{remark}[Feasibility boundary]
We analyze the dynamics only on the violation region $\mathcal{D} := \{\bm{x} : \mathcal{V}(\bm{x}) > 0\}$, where the penalty $\mathcal{V}$ is differentiable. Once $\mathcal{V}(\bm{x}) = 0$ is reached, the constraints are satisfied and the procedure terminates.  This technical restriction avoids the nondifferentiability of the ReLU penalty at~0.
\end{remark}

\begin{theorem}[Asymptotic Convergence of Integer Feasibility Projection]
\label{thm:relu_exterior_convergence}
Under Assumptions~\ref{assump:regularity} and~\ref{assump:bounded_constraints}, gradient descent applied to 
$
\mathcal{V}({\bm{x}}) = \|{\bm g}(\varphi({\bm{x}}))_+\|_1
$
with step size $\eta \in (0, 1/L]$, where $L = \bar{n}_c(G_{\bm g} L_\varphi + G_\varphi L_{\bm g})$, satisfies the following properties:
\begin{enumerate}
    \item[\textnormal{(i)}] \textbf{L-smoothness:} 
    $\mathcal{V} \in C^1(\mathcal{D})$, with
    $
    \nabla \mathcal{V}({\bm{x}}) 
    = \sum_{j \in I_{\bm{x}}} 
    \nabla \varphi({\bm{x}})^\top \nabla g_j(\varphi({\bm{x}})),
    $
    and $\nabla \mathcal{V}$ is Lipschitz continuous on compact subsets of $\mathcal{D}$ 
    with Lipschitz constant at most $L$.
    
    \item[\textnormal{(ii)}] \textbf{Descent and vanishing gradient:} 
    Gradient descent generates a non-increasing sequence $\mathcal{V}({\bm{x}}^{(k)}) \to \mathcal{V}^\star \ge 0$, and
    $
    \lim_{k \to \infty} 
    \| \nabla \mathcal{V}({\bm{x}}^{(k)}) \| = 0.
    $
    
    \item[\textnormal{(iii)}] \textbf{Convergence to feasibility:} 
    If every ${\bm{x}}^* \in \mathcal{D}$ with $\mathcal{V}({\bm{x}}^*) > 0$ satisfies $\exists j \in I_{{\bm{x}}^*}$ such that $\nabla g_j(\varphi({\bm{x}}^*)) \ne 0,$ then
    $
    \lim_{k \to \infty} 
    \mathcal{V}({\bm{x}}^{(k)}) = 0.
    $
\end{enumerate}
\end{theorem}

We now establish the asymptotic convergence guarantees of the integer feasibility projection. 

\begin{proof}
{Proof of ~\Cref{thm:relu_exterior_convergence}.} We prove~\Cref{thm:relu_exterior_convergence} in the following steps.

\paragraph{L-smoothness of $\nabla \mathcal{V}$ on compact subsets of $\mathcal{D}$.}
Define the violation region as
$
\mathcal{D} := \{{\bm x} \in \mathbb{R}^n : \mathcal{V}({\bm x}) > 0\},
$
which consists of all points violating at least one constraint. 
The penalty function is
$$
\mathcal{V}({\bm x}) = \|{\bm g}(\varphi({\bm x}))_+\|_1 
= \sum_{j=1}^{n_c} \max(0, g_j(\varphi({\bm x}))).
$$
For each ${\bm x} \in \mathcal{D}$, define the active set
$
I_{\bm x} := \{ j : g_j(\varphi({\bm x})) > 0 \}.
$
On this set, $\max(0, g_j(\varphi({\bm x}))) = g_j(\varphi({\bm x}))$, 
which is smooth since both $g_j$ and $\varphi$ are continuously differentiable. 
Thus, we can write
$
\mathcal{V}({\bm x}) = \sum_{j \in I_{\bm x}} g_j(\varphi({\bm x})),
$
and by the chain rule,
$
\nabla \mathcal{V}({\bm x}) 
= \sum_{j \in I_{\bm x}} 
\nabla \varphi({\bm x})^\top \nabla g_j(\varphi({\bm x})).
$

Let ${\bm x}_1, {\bm x}_2 \in \mathcal{D}$ belong to a compact subset $\mathcal{K} \subset \mathcal{D}$, 
and define $I := I_{{\bm x}_1} \cup I_{{\bm x}_2}$.
By the triangle inequality and submultiplicativity of norms,
$$
\begin{aligned}
\|\nabla \mathcal{V}({\bm x}_1) - \nabla \mathcal{V}({\bm x}_2)\|
\le 
\sum_{j \in I}
\|
\nabla \varphi({\bm x}_1)^\top \nabla g_j(\varphi({\bm x}_1))
- \nabla \varphi({\bm x}_2)^\top \nabla g_j(\varphi({\bm x}_2))
\|
\\
\le 
\sum_{j \in I}
\Big(
\|\nabla \varphi({\bm x}_1) - \nabla \varphi({\bm x}_2)\| \, \|\nabla g_j(\varphi({\bm x}_1))\|
 +
\|\nabla \varphi({\bm x}_2)\| \,
\|\nabla g_j(\varphi({\bm x}_1)) - \nabla g_j(\varphi({\bm x}_2))\|
\Big).
\end{aligned}
$$
Applying Assumption~\ref{assump:regularity} gives
$
\|\nabla \varphi({\bm x}_1) - \nabla \varphi({\bm x}_2)\| \le L_\varphi \|{\bm x}_1 - {\bm x}_2\|,  \,\,\,\,
\|\nabla g_j(\varphi({\bm x}_1)) - \nabla g_j(\varphi({\bm x}_2))\| \le L_{\bm g} G_\varphi \|{\bm x}_1 - {\bm x}_2\|, \,\,\,\,
\|\nabla g_j(\varphi({\bm x}))\| \le G_{\bm g}, \,\,\,\, \|\nabla \varphi({\bm x})\| \le G_\varphi.
$
Combining these bounds yields
$$
\|\nabla \mathcal{V}({\bm x}_1) - \nabla \mathcal{V}({\bm x}_2)\|
\le
|I| (G_{\bm g} L_\varphi + G_\varphi L_{\bm g}) \|{\bm x}_1 - {\bm x}_2\|.
$$
Hence, $\nabla \mathcal{V}$ is Lipschitz continuous on compact subsets of $\mathcal{D}$ with local constant 
$L({\bm x}_1, {\bm x}_2) = |I| (G_{\bm g} L_\varphi + G_\varphi L_{\bm g})$. 
When the number of active constraints is uniformly bounded by $\bar{n}_c$, the global Lipschitz constant simplifies to
$
L = \bar{n}_c(G_{\bm g} L_\varphi + G_\varphi L_{\bm g}),
$
which is substantially tighter than the naive worst-case bound $L = n_c(G_{\bm g} L_\varphi + G_\varphi L_{\bm g})$.

\paragraph{Descent lemma and vanishing gradient norm.}
Having established that $\nabla \mathcal{V}$ is $L$-Lipschitz continuous on $\mathcal{D}$,
we now invoke the standard descent lemma for smooth functions. 
It implies that
$$
\mathcal{V}({\bm x}^{(k+1)}) 
\le 
\mathcal{V}({\bm x}^{(k)}) 
- \eta \left(1 - \frac{L\eta}{2}\right) 
\| \nabla \mathcal{V}({\bm x}^{(k)}) \|^2.
$$
If $\eta \in (0, 1/L]$, then $1 - \frac{L\eta}{2} \ge 1/2$, and hence
$
\mathcal{V}({\bm x}^{(k+1)}) 
\le 
\mathcal{V}({\bm x}^{(k)}) 
- \frac{\eta}{2} 
\| \nabla \mathcal{V}({\bm x}^{(k)}) \|^2.
$
Thus, the sequence $\{\mathcal{V}({\bm x}^{(k)})\}$ is monotonically non-increasing 
and bounded below by zero, hence convergent to some finite limit $\mathcal{V}^\star \ge 0$.

Summing from $k = 0$ to $K-1$ gives
$
\sum_{k=0}^{K-1} 
\| \nabla \mathcal{V}({\bm x}^{(k)}) \|^2 
\le 
\frac{2}{\eta}
\left[ 
\mathcal{V}({\bm x}^{(0)}) - \mathcal{V}({\bm x}^{(K)})
\right]
\le 
\frac{2}{\eta} \mathcal{V}({\bm x}^{(0)}).
$

Therefore,
$
\sum_{k=0}^{\infty} 
\| \nabla \mathcal{V}({\bm x}^{(k)}) \|^2 
< \infty 
\quad \Rightarrow \quad 
\lim_{k \to \infty} 
\| \nabla \mathcal{V}({\bm x}^{(k)}) \| = 0.
$

\paragraph{Convergence to feasibility.}
Assume for contradiction that $\mathcal{V}^\star > 0$. 
Then there exists a convergent subsequence $\{{\bm x}^{(k_j)}\}$ 
with limit ${\bm x}^* \in \mathcal{D}$ such that 
$\mathcal{V}({\bm x}^*) > 0$ and $\nabla \mathcal{V}({\bm x}^*) = 0$.
Since the ReLU penalty is smooth wherever $g_j(\varphi({\bm x}^*)) > 0$, 
the composite gradient at ${\bm x}^*$ is
$$
\nabla \mathcal{V}({\bm x}^*) 
= 
\sum_{j \in I_{{\bm x}^*}} 
\nabla \varphi({\bm x}^*)^\top \nabla g_j(\varphi({\bm x}^*)).
$$
If at least one active constraint $j \in I_{{\bm x}^*}$ satisfies 
$\nabla g_j(\varphi({\bm x}^*)) \ne 0$ and 
$\nabla \varphi({\bm x}^*) \ne 0$, 
then the corresponding inner product is generically nonzero, 
implying $\nabla \mathcal{V}({\bm x}^*) \ne 0$—a contradiction to stationarity.

While theoretical cancellations or orthogonality between 
$\nabla \varphi({\bm x}^*)$ and all $\nabla g_j$ could yield 
$\nabla \mathcal{V}({\bm x}^*) = 0$, such configurations are nongeneric 
and do not correspond to stable local minima of $\mathcal{V}$, 
as discussed in Theorem~\ref{thm:local_minima_h}. 
In particular, points with $\nabla \varphi({\bm x}^*) = 0$ 
arise from flat regions of the rounding map 
and are ruled out as attractors by the Łojasiewicz descent framework 
in Theorem~\ref{thm:loja_convergence}.

Therefore, no infeasible critical point can be a stable limit point 
of the projected gradient iterates. 
Consequently, the sequence converges to the feasible boundary:
$
\lim_{k \to \infty} \mathcal{V}({\bm x}^{(k)}) = 0.
$  \hfill  $\square$
\end{proof}

\subsection{Non-Asymptotic Convergence of Integer Feasibility Projection}

Building upon the asymptotic analysis, we now establish a finite-time convergence rate for the integer feasibility projection. This result characterizes how fast the constraint violation $\mathcal{V}(\bm{x})$ decreases under gradient descent with a fixed step size.

\begin{theorem}[Non-Asymptotic Convergence of Integer Feasibility Projection]
\label{thm:non_asymptotic_relu}
Under the assumptions of \Cref{thm:relu_exterior_convergence}, suppose gradient descent is applied to the function
$ \mathcal{V}({\bm{x}}) = \sum_{j=1}^{n_c} \max(0, g_j(\varphi({\bm{x}}))),
$
with fixed step size $\eta \in \left(0, \frac{1}{L} \right]$, where
$ L := \bar{n}_c  (G_{\bm g} L_\varphi + G_\varphi L_{\bm g})$
is an upper bound on the Lipschitz constant of $\nabla \mathcal{V}$ over the region $\mathcal{D} := \{ x : \mathcal{V}({\bm{x}}) > 0 \}$.
Then for any number of iterations $K \ge 1$, the minimum gradient norm over the first $K$ iterates satisfies
$ 
\min_{0 \le k < K} \| \nabla \mathcal{V}({\bm{x}}^{(k)}) \|^2 \le \frac{2}{\eta K} \left[ \mathcal{V}({\bm{x}}^{(0)}) - \mathcal{V}^\star \right],
$
where $\mathcal{V}^\star := \inf_{x \in \mathcal{D}} \mathcal{V}({\bm{x}}) \ge 0$.
In particular, to ensure
$\min_{0 \le k < K} \| \nabla \mathcal{V}({\bm{x}}^{(k)}) \| \le \delta $,
it suffices to run
$
K \ge \frac{2}{\eta \delta^2} (\mathcal{V}({\bm{x}}^{(0)}) - \mathcal{V}^\star)$
iterations with complexity $K = \mathcal{O} \left( \frac{1}{\delta^2} \right)$.
Furthermore, if $\mathcal{V}^\star = 0$, then for any  $\epsilon > 0$ this implies approximate feasibility
$ \mathcal{V}({\bm{x}}^{(k)}) < \epsilon
\quad \text{for all } k \ge K_\epsilon \text{ for some } K_\epsilon$.
\end{theorem}

\begin{proof}
{Proof of ~\Cref{thm:non_asymptotic_relu}.} Since $\mathcal{V}$ is differentiable with $L$-Lipschitz gradient and gradient descent stepsize is $\eta \le 1/L$, the standard descent lemma implies
$
\mathcal{V}(\bm{x}^{(k+1)}) 
\le 
\mathcal{V}(\bm{x}^{(k)}) 
- \frac{\eta}{2} \| \nabla \mathcal{V}(\bm{x}^{(k)}) \|^2.
$

Summing from $k = 0$ to $K - 1$ gives
$
\mathcal{V}(\bm{x}^{(0)}) - \mathcal{V}(\bm{x}^{(K)}) 
\ge 
\frac{\eta}{2} \sum_{k=0}^{K-1} \| \nabla \mathcal{V}(\bm{x}^{(k)}) \|^2.
$

Since $\mathcal{V}(\bm{x}^{(K)}) \ge \mathcal{V}^\star$, we obtain
$
\sum_{k=0}^{K-1} \| \nabla \mathcal{V}(\bm{x}^{(k)}) \|^2 
\le 
\frac{2}{\eta} \left[ \mathcal{V}(\bm{x}^{(0)}) - \mathcal{V}^\star \right].
$

Dividing both sides by $K$ yields
$
\min_{0 \le k < K} \| \nabla \mathcal{V}(\bm{x}^{(k)}) \|^2 
\le 
\frac{2}{\eta K} \left[ \mathcal{V}(\bm{x}^{(0)}) - \mathcal{V}^\star \right].
$

Finally, if $\mathcal{V}^\star = 0$, then $\mathcal{V}(\bm{x}^{(k)}) \to 0$, implying that for any $\epsilon > 0$, there exists $K_\epsilon$ such that $\mathcal{V}(\bm{x}^{(k)}) < \epsilon$ for all $k \ge K_\epsilon$.
\hfill  $\square$
\end{proof}

\begin{remark}
Theorem~\ref{thm:non_asymptotic_relu} establishes convergence to an approximate first-order stationary point of the composite constraint violation function
$
\mathcal{V}(\bm{x}) = \sum_j \max(0, g_j(\varphi(\bm{x})))
$,
at rate $\mathcal{O}(1/K)$, under standard smoothness and boundedness assumptions.
This ensures that the projected gradient method converges to a critical point of this nonconvex function, although the corresponding hard-rounded point $\varphi(\bm{x})$ may not yet be strictly feasible.
However, Theorem~\ref{thm:local_minima_h} shows that infeasible stationary points are nongeneric and cannot correspond to local minima. 
Taken together, Theorems~\ref{thm:non_asymptotic_relu} and~\ref{thm:local_minima_h} guarantee convergence toward approximately feasible integer solutions in practice.
\end{remark}

\begin{corollary}[Iteration Complexity for Approximate Feasibility]
\label{cor:approx_feasibility}
Suppose the conditions of Theorem~\ref{thm:non_asymptotic_relu} hold and that $\mathcal{V}^\star = 0$.
Then for any tolerance $\epsilon > 0$, gradient descent with $\eta \in (0, 1/L]$ produces an iterate $\bm{x}^{(k)}$ satisfying
$
\mathcal{V}(\bm{x}^{(k)}) < \epsilon
$,
after at most
$
K_\epsilon := \left\lceil \frac{2}{\eta \epsilon} \mathcal{V}(\bm{x}^{(0)}) \right\rceil
$
iterations. Hence,
$
\bm{x}^{(k)} \in S_\epsilon := \left\{ \bm{x} \in \mathbb{R}^n : \mathcal{V}(\bm{x}) < \epsilon \right\},
\, \forall k \ge K_\epsilon.
$
\end{corollary}

\begin{remark}[On Regularity and Practical Tightness of Lipschitz Assumptions]
\label{rmk:regularity_lipschitz}
The analyses in Theorems~\ref{thm:relu_exterior_convergence} and~\ref{thm:non_asymptotic_relu} rely on the assumption that both $\bm{g}$ and $\varphi$ are $C^1$ with Lipschitz continuous gradients and bounded Jacobians. These conditions are realistic in practice.
The ReLU penalty is smooth on the exterior domain $\mathcal{D} = \{\bm{x} : \mathcal{V}(\bm{x}) > 0\}$ and thus preserves the differentiability of $\mathcal{V}$.
    The constraint map $\bm{g}: \mathbb{R}^n \to \mathbb{R}^{n_c}$ typically consists of smooth nonlinear functions defined over compact domains such as $[0,1]^n$. Hence, $\bm{g} \in C^2$, and its Jacobian $\nabla \bm{g}$ is Lipschitz continuous and bounded. 
    The surrogate rounding $\varphi$ is parameterized by smooth functions (scaled sigmoid or Gumbel-sigmoid) with finite, strictly positive temperature parameters. In this case, $\varphi \in C^2$ with bounded and Lipschitz continuous gradients.
The bound 
$
L := \bar{n}_c (G_{\bm{g}} L_\varphi + G_\varphi L_{\bm{g}})
$
is generally conservative. 
In many practical cases, $\bm{g}(\varphi(\bm{x}))$ is sparse or low-rank, and each $g_j$ depends only on a small subset of variables, yielding significantly smaller Lipschitz constants. 
Exploiting such sparsity can lead to tighter complexity bounds and faster empirical convergence.
\end{remark}

\subsection{Extended Convergence Analysis}

Because the penalty $\mathcal{V}(\bm{x})$ is subanalytic and satisfies the Kurdyka--Łojasiewicz (KŁ) inequality~\citep{bolte2007_clarke}, the basic descent argument extends naturally to nonconvex settings: the iterates $\{\bm{x}^{(k)}\}$ produced by the feasibility projection converge to a stationary point, and $\mathcal{V}(\bm{x}^{(k)}) \to mathcal{V}^\star$. Appendix~\ref{subapp:proof_Łoj} provides the corresponding technical derivation.

Under mild regularity conditions, the penalty landscape admits no persistent infeasible local minima; infeasible critical points are either ruled out by non-vanishing constraint gradients or are dynamically unstable. This ensures that the projection mechanism continues to move iterates toward feasibility (see Appendix~\ref{subapp:proof_local}).

\begin{remark}[Robustness]
Appendix~\ref{subapp:degen} further shows that degeneracies such as plateaus, non-isolated critical points, and flat manifolds are nongeneric and do not attract gradient-descent trajectories, explaining the stable empirical behavior of the projection step.
\end{remark}

These guarantees are local rather than global: Proving global feasibility for arbitrary MINLPs is intractable. Our analysis focuses on the local convergence behavior of the integer feasibility projection under standard smoothness and regularity assumptions. In practice, however, the correction layer typically produces initial points near the feasible region, and we observe trajectories consistent with the above theory.

\section{Benchmark Problems}

We evaluate our framework on three representative families of mixed-integer optimization problems: integer quadratic problems (IQPs), integer nonconvex problems (INPs), and mixed-integer Rosenbrock problems (MIRBs). These benchmarks span diverse structures and nonlinearities, and problem sizes from tens to tens of thousands of variables.

\paragraph{Integer Quadratic Problems.} 
The IQPs are formulated as
$$
\min_{{\bm{x}} \in \mathbb{Z}^n}
\;\frac{1}{2}\bm{x}^\top \bm{Q} \bm{x} + \bm{p}^\top \bm{x}
\quad\text{s.t.}\quad
\bm{A}\bm{x} \le \bm{b},
$$

Following \citet{donti2021dc3}, we adopt a similar data-generation scheme. Complete data-generation details are provided in Appendix~\ref{subapp:qp_gen}.

\paragraph{Integer Non-convex Problems.} 
Following \citet{donti2021dc3}, INPs extend the IQP formulation:
$$
\min_{{\bm{x}} \in \mathbb{Z}^n} \ \frac{1}{2} {\bm{x}}^{\intercal} {\bm Q} {\bm{x}} + {\bm p}^{\intercal} \sin{({\bm{x}})}
\quad \text{subject to} \quad {\bm A} {\bm{x}} \leq {\bm b},
$$

Compared to IQPs, INPs exhibit nonconvex objective landscapes and instance-dependent feasible regions.
The full instance construction is described in Appendix~\ref{subapp:nc_gen}.

\paragraph{Mixed-integer Rosenbrock Problems.}
The mixed-integer Rosenbrock problems (MIRBs) are defined as:
$$
\min_{{\bm x} \in \mathbb{R}^n, {\bm y} \in \mathbb{Z}^n} 
\ \|{\bm a} - {\bm x}\|_2^2 + 50 \| {\bm y} - {\bm x}^2 \|_2^2 \quad 
\text{subject to} \quad
\| {\bm{x}} \|_2^2 \leq n b, \quad
\bm{1}^\intercal {\bm y} \geq \frac{n b}{2}, \quad
{\bm p}^\intercal {\bm x} \leq 0, \quad
{\bm Q}^\intercal {\bm y} \leq 0,
$$

Here, $(\bm{x},\bm{y})$ denote continuous and integer variables, respectively, while problem parameters vary across instances. Further details are given in Appendix~\ref{subapp:rb_gen}.

\section{Experimental Results}

We evaluate the proposed framework on a range of mixed-integer optimization problems, assessing (i) the effectiveness of the correction layers, (ii) the impact of feasibility projection, and (iii) performance relative to classical heuristics and exact solvers.

\subsection{Experimental Setup}
\label{subsec:setup}

\paragraph{Implementation and Training.}
All learning-based models are trained on 8{,}000 instances, validated on 1{,}000, 
and tested on 100 unseen samples. Training is performed using PyTorch and 
NeuroMANCER~\citep{Neuromancer2023} on a GPU workstation. Solver-based baselines use \gurobi{} for convex 
problems and \scip{}+\ipopt{} for nonconvex ones under identical CPU settings. Each learning-based method employs a feedforward architecture consisting of a 
solution-mapping module followed by an integer correction module. Hidden widths 
and architectural choices scale with the problem dimension, and differ across 
datasets. 
Full hyperparameters, network specifications, solver configuration, 
software environments, and hardware specifications are provided in Appendix~\ref{app:exp_details}.

\paragraph{Methods.}
\Cref{tab:methods} summarizes all methods evaluated in our experiments. A uniform time limit of 1000 seconds is applied to every solver and method. We assess two learning-based approaches, RC and LT, which directly predict integer solutions through differentiable correction layers. Their enhanced counterparts, RC-P and LT-P, further incorporate the integer feasibility projection to improve constraint satisfaction. 
We include a range of baselines: (i) \textit{Exact solvers (EX):} Implemented with \gurobi{} for convex problems and \scip{} for nonconvex problems. These methods guarantee optimality when tractable but incur high computational cost. (ii) Heuristic baselines: \textit{Rounding after Relaxation (RR)} directly rounds the continuous relaxation, while \textit{Root Node Solution (N1)} extracts the first feasible solution from solver. Both EX and N1 implicitly rely on solver-integrated heuristics (e.g., presolve, primal heuristics, cut management), so our comparison implicitly reflects the effectiveness of these built-in heuristics versus our learning-based alternatives.
Overall, these methods span a broad spectrum from exact solvers to data-driven and heuristic strategies, enabling a comprehensive evaluation across feasibility, solution quality, and runtime. Our implementation is publicly available at \href{https://github.com/pnnl/L2O-pMINLP}{\texttt{github.com/pnnl/L2O-pMINLP}}.

\begin{table*}[htb!]
\centering
\footnotesize
\setlength{\tabcolsep}{6pt}
\renewcommand{\arraystretch}{0.85}
\caption{Summary of methods (``*'' indicates a trained model).}
\label{tab:methods}

\begin{tabular}{>{\raggedright\arraybackslash}m{4.2cm} > {\raggedright\arraybackslash}m{6cm}}
\toprule
\textbf{Method} & \textbf{Description} \\
\midrule
Rounding Classification (RC)* 
& Learns rounding directions by predicting probabilities for rounding up or down.\\
\midrule

RC with Projection (RC-P)* 
& Adds an integer-feasibility projection at inference time for RC.\\
\midrule

Learnable Threshold (LT)* 
& Predicts instance-dependent thresholds determining rounding directions.\\
\midrule

LT with Projection (LT-P)* 
& Adds an integer-feasibility projection at inference time for  LT.\\
\midrule

Exact Solver (EX) 
& Solves each instance using commercial solvers such as \gurobi{} or \scip{}.\\
\midrule

Rounding after Relaxation (RR)
& Solves the continuous relaxation and rounds integer variables.\\
\midrule

Root Node Solution (N1)
& Uses the first feasible integer solution found at the root node of branch-and-cut.\\


\bottomrule
\end{tabular}
\end{table*}


\subsection{Qualitative Behavior}
\label{subsec:qual}

\paragraph{Solver-Time Comparison.}
As illustrated in~\Cref{fig:solving}, exact solvers such as \gurobi{} find better solutions over time but can be slow. For more complex problem instances, these solvers may fail to find feasible solutions within strict time limits. In contrast, our proposed methods consistently achieve high-quality feasible solutions within milliseconds. To the best of our knowledge, this is the first general approach for efficiently solving parametric MINLPs with up to tens of thousands of variables.

\begin{figure}[htb!]
\vspace{-6pt}
    \centering
    \begin{minipage}{0.44\textwidth}
        \centering        \includegraphics[width=\textwidth]{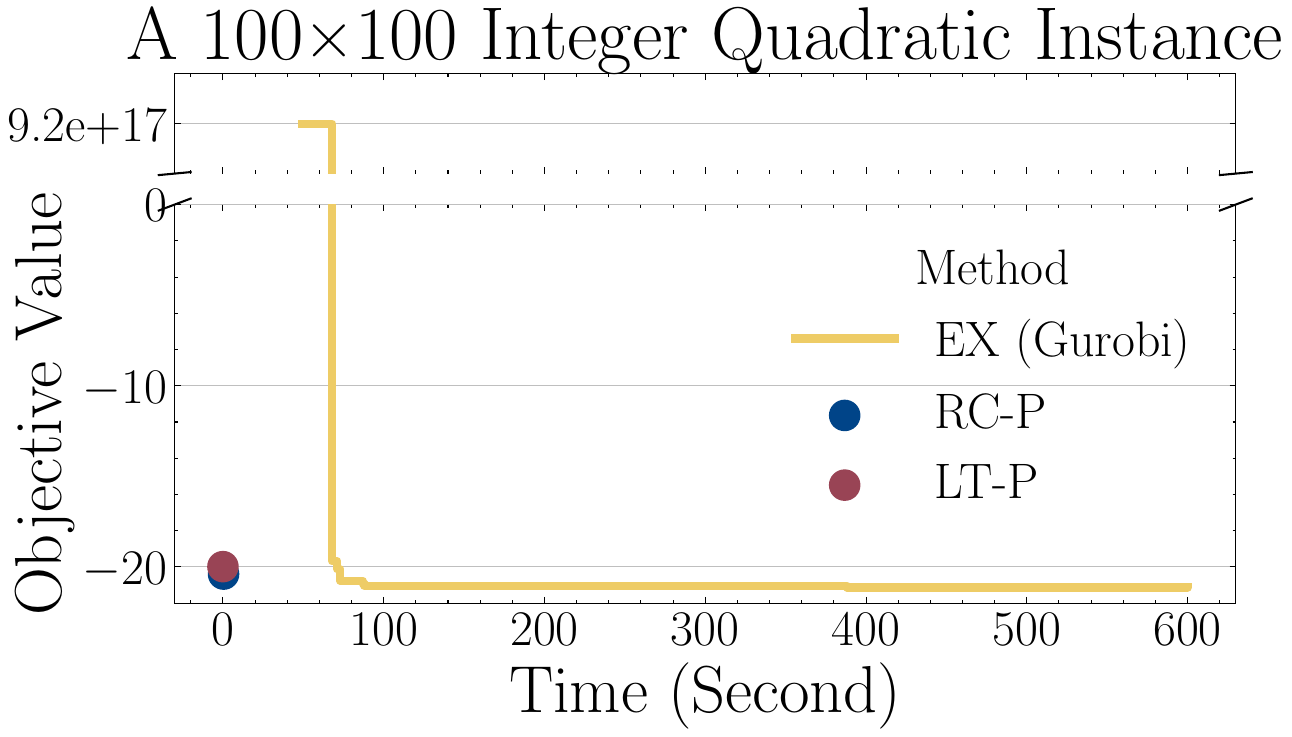}
        \label{fig:qc_sol}        \end{minipage}\hspace{0.05\textwidth}
    \begin{minipage}{0.44\textwidth}
        \centering        \includegraphics[width=\textwidth]{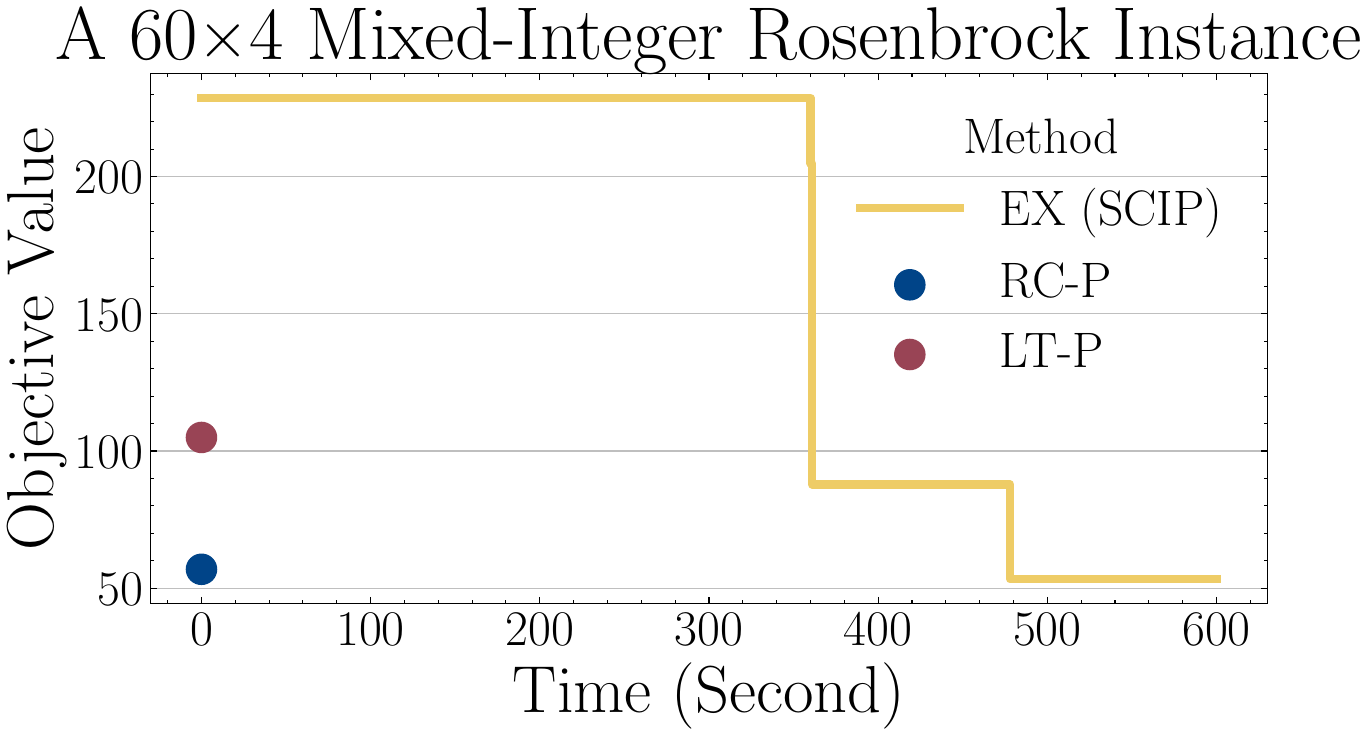}
        \label{fig:rb_sol}
    \end{minipage}    \caption{Illustration of objective value evolution for a $100 \times 100$  Integer Quadratic and $60 \times 4$ Mixed-Integer Rosenbrock over 600 seconds.}
    \label{fig:solving}
\vspace{-6pt}
\end{figure}

Even when accounting for training time (100 seconds), the overall efficiency of RC and LT remains substantially better. Importantly, once trained, the models effectively generalize to unseen problem instances, making them ideal for repeated problem-solving scenarios where the training cost is amortized~\cite{Amos_amortized2022}. Furthermore, RC and LT can generate high-quality initial solutions for exact solvers, reducing the search space and accelerating the convergence of traditional methods.

\paragraph{Case Study: 2D Mixed-Integer Rosenbrock}
To provide intuition for how our framework operates in practice, we first illustrate the interplay between the integer correction and feasibility projection modules on a simple two-dimensional example. Specifically, we consider a Mixed-Integer Rosenbrock Benchmark (MIRB) instance, formulated as:
$$
\min_{x \in \mathbb{R}, y \in \mathbb{Z}} \quad (a - x)^2 + 50(y - x^2)^2 \quad 
\text{subject to} \quad y \geq {b}/{2}, \quad x^2 \leq b, \quad x \leq 0, \quad y \geq 0. 
$$
Here, $x$ is a continuous decision variable and $y$ is an integer decision variable, both subject to linear constraints. The instance parameters $a$ and $b$ serve as input features to the neural network.

\begin{figure}[htbp!]
\vspace{-6pt}
    \centering 
    \begin{minipage}{0.35\textwidth}
    \includegraphics[width=\textwidth]{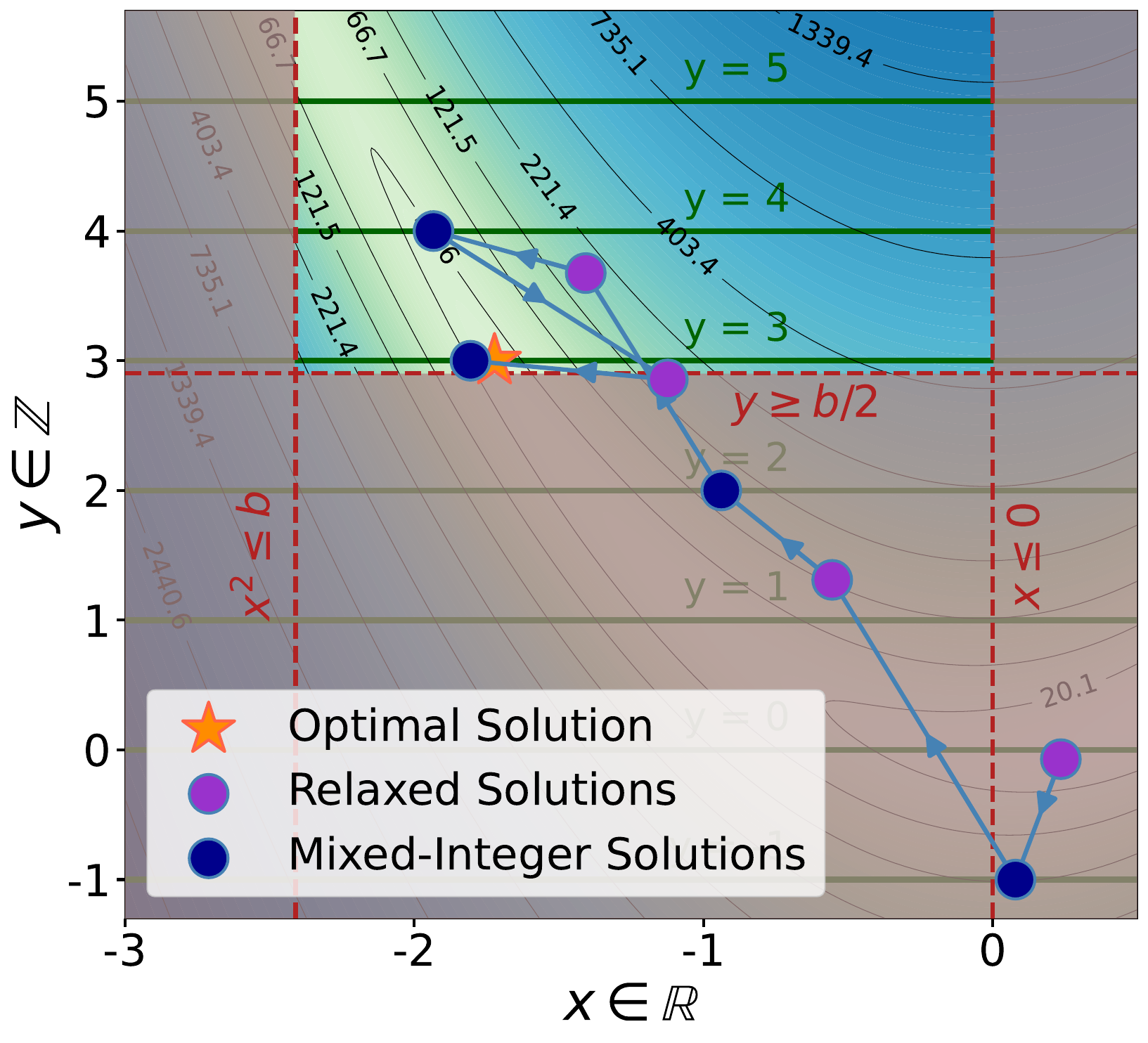}
    \caption{Example of the relaxed solutions $\bar{x}, \bar{y}$ and the mixed-integer solutions $\hat{x}, \hat{y}$ across different epochs of training for the same sample instance.}
    \label{fig:training}
    \end{minipage}
    \hspace{0.02\textwidth}
    \begin{minipage}{0.35\textwidth}
    \vspace{-0.08cm}
    \includegraphics[width=\textwidth]{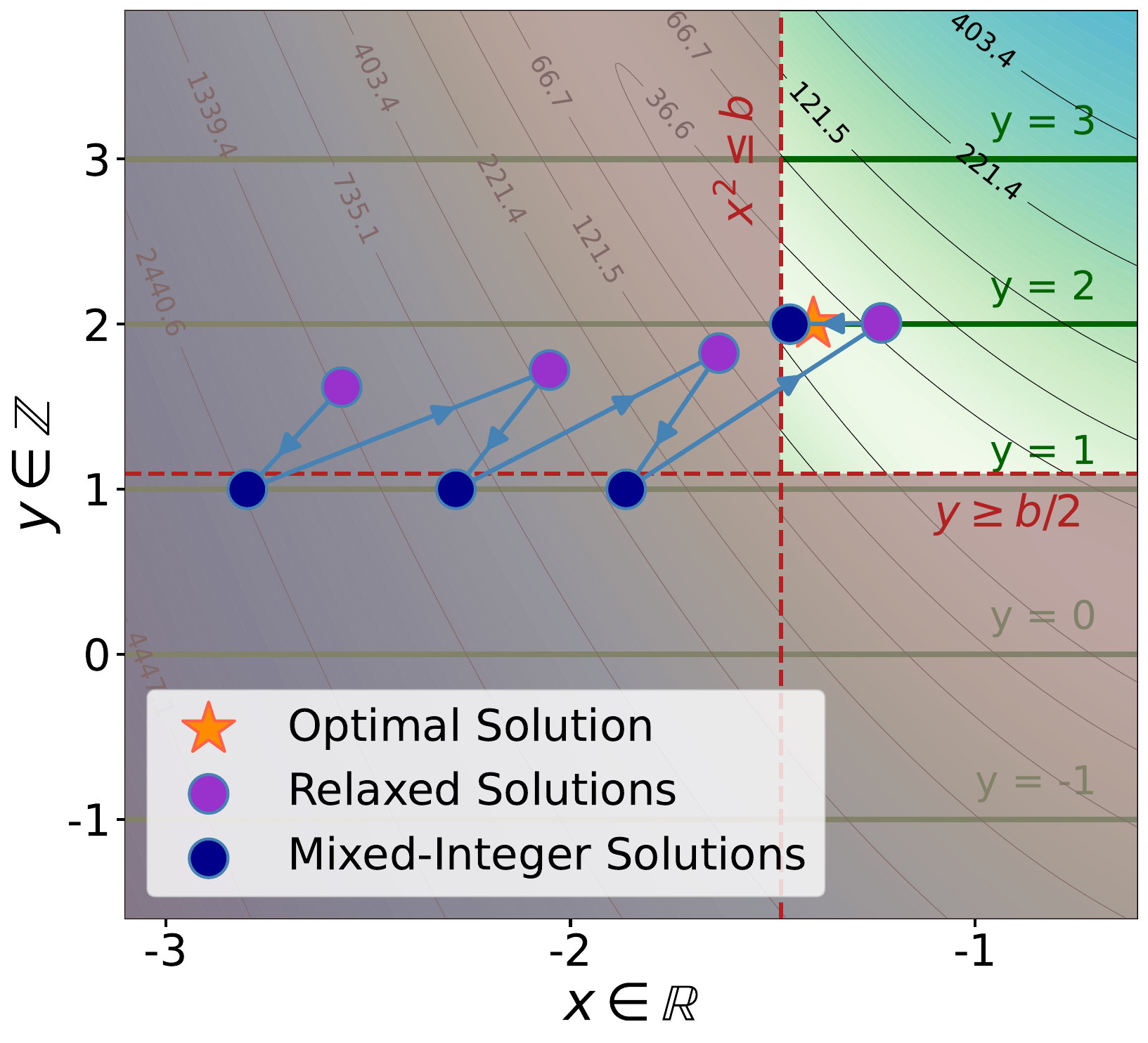}
    \caption{Example of the relaxed solutions $\bar{x}, \bar{y}$ and the mixed-integer solutions $\hat{x}, \hat{y}$ across iterations of feasibility projection to refine an infeasible solution.}
    \label{fig:projection}
    \end{minipage}
\vspace{-12pt}
\end{figure}

\Cref{fig:training} illustrates the evolution of relaxed solutions $(\bar{x}, \bar{y})$ and their corresponding mixed-integer solutions $(\hat{x}, \hat{y})$ over successive training epochs. In this example, the instance parameters are set to $a = 3.83$ and $b = 6.04$. During training, the neural network progressively refines the relaxed solutions $(\bar{x}, \bar{y})$, while the integer correction layer maps them to mixed-integer counterparts $(\hat{x}, \hat{y})$. As learning proceeds, the predictions gradually approach feasible and near-optimal regions, demonstrating that the correction layer effectively enforces integer feasibility without sacrificing objective quality or constraint satisfaction.

\Cref{fig:projection} illustrates the iterative refinement performed by the integer feasibility projection module on an initially infeasible solution. In this example, the instance parameters are $a = 4.16$ and $b = 2.19$. Starting from the relaxed solution $(\bar{x}, \bar{y})$, the projection procedure applies gradient-based updates that progressively reduce constraint violations while preserving the integer feasibility enforced by the correction layer. As the iterations proceed, the solution converges toward the feasible region, and the final mixed-integer output $(\hat{x}, \hat{y})$ satisfies all constraints.

While the main text reports aggregate feasibility statistics, the heatmaps below provide instance-level visualizations of constraint violations for the largest benchmark sizes used in our experiments: $500\times500$ for IQP, $500\times500$ for INP, and $20000\times 4$ for MIRB. These plots illustrate typical violation patterns corrected by RC and LT across 100 test instances for each class.

\subsection{Main Experiments}

We evaluate the proposed methods on a broad range of problem instances. For IQPs and INPs, we consider sizes from $20\times20$ (20 variables and 20 constraints) up to $1000\times1000$, while for MIRBs we experiment with instances containing between $2$ and $20{,}000$ decision variables with the number of constraints fixed at $4$. Results for these three problem classes are summarized in~\Cref{tab:QP_Stat}, \Cref{tab:NC_Stat}, and \Cref{tab:RB_Stat}. For methods that rely on exact solvers (EX, N1, and RR), the solver may fail to return a solution within the $1{,}000$ second time-limit; to quantify this, we report the percentage of instances solved (``\%Solved'').
To assess the generality of our approach beyond nonlinear settings, we additionally evaluate RC and LT on binary linear programs. Consistent performance trends are observed, and detailed experimental results are provided in Appendix~\ref{app:blps}.

\begin{table}[ht!]
\centering
\caption{
\textbf{Results for IQPs.}  
Each problem size is evaluated on a test set of 100 instances. ``Obj Mean'' and ``Obj Med'' represent the mean and median objective values for this minimization problem, with smaller values being better. ``Feasible'' denotes the fraction of feasible solutions, ``Solved'' denotes the percentage of instances where a solution (feasible or infeasible) was found within the time limit, and ``Time'' is the average solving/inference time per instance. ``—'' indicates that no solution was found within 1000 seconds. For methods achieving 100\% feasibility, we highlight in bold the best-performing metrics.}
\label{tab:QP_Stat}
\resizebox{0.8\linewidth}{!}{
\begin{tabular}{ll|rrrrrr}
\toprule[1pt]\midrule[0.3pt]
\textbf{} & \textbf{Metric} & \textbf{20×20} & \textbf{50×50} & \textbf{100×100} & \textbf{200×200} & \textbf{500×500}  & \textbf{1000×1000} \\
\midrule
\multirow{4}{*}{RC}
    & Obj Mean         & $-4.237$  & $-12.20$   & $-13.54$   & $-31.62$   & $-73.31$   & $-142.7$    \\
    & Obj Med       & $-4.307$  & $-12.20$   & $-13.60$   & $-31.71$   & $-73.38$   & $-142.7$    \\
    & Feasible      & $99\%$    & $99\%$      & $96\%$    & $97\%$     & $86\%$     & $82\%$      \\
    & Time       & $0.0019$  & $0.0019$   & $0.0022$   & $0.0021$   & $0.0025$   & $0.0042$    \\
\midrule
\multirow{4}{*}{RC-P} 
    & Obj Mean         & $-4.238$  & $-12.20$   & $-13.54$   & $\bm{-31.62}$   & $\bm{-73.31}$   & $\bm{-142.7}$    \\
    & Obj Med       & $-4.307$  & $-12.20$   & $-13.57$   & $\bm{-31.71}$   & $\bm{-73.38}$   & $\bm{-142.7}$    \\
    & Feasible      & $100\%$   & $100\%$    & $100\%$    & $100\%$    & $100\%$    & $100\%$     \\
    & Time       & $\bm{0.0045}$  & $0.0055$   & $\bm{0.0050}$   & $\bm{0.0050}$   & $0.0065$   & $0.0090$    \\
\midrule
\multirow{4}{*}{LT} 
    & Obj Mean         & $-4.302$  & $-12.98$   & $-13.65$   & $-31.34$   & $-72.36$   & $-142.6$    \\
    & Obj Med       & $-4.319$  & $-13.03$   & $-13.77$   & $-31.61$   & $-72.48$   & $-142.6$    \\
    & Feasible      & $98\%$    & $98\%$     & $93\%$     & $95\%$     & $94\%$     & $100\%$     \\
    & Time       & $0.0020$  & $0.0020$   & $0.0023$   & $0.0022$   & $0.0026$   & $0.0047$    \\
\midrule
\multirow{4}{*}{LT-P} 
    & Obj Mean         & $-4.301$  & $-12.98$   & $-13.65$   & $-31.34$   & $-72.36$   & $-142.6$    \\
    & Obj Med       & $-4.316$  & $-13.03$   & $-13.77$   & $-31.61$   & $-72.48$   & $-142.6$    \\
    & Feasible      & $100\%$   & $100\%$    & $100\%$    & $100\%$    & $100\%$    & $100\%$       \\
    & Time       & $0.0056$  & $\bm{0.0055}$   & $0.0100$   & $0.0064$   & $\bm{0.0063}$   & $\bm{0.0086}$    \\
\midrule
\multirow{4}{*}{EX} 
    & Obj Mean         & $\bm{-5.120}$  & $\bm{-15.93}$  & $\bm{-20.79}$    & —          & —          & —           \\
    & Obj Med       & $\bm{-5.130}$  & $\bm{-15.96}$  & $\bm{-20.78}$    & —          & —          & —           \\
    & Feasible    & $100\%$   & $100\%$   & $100\%$     & —          & —          & —           \\
    & Solved        & $100\%$   & $100\%$    & $100\%$    & $0\%$          & $0\%$           & $0\%$            \\
    & Time       & $8.728$   & $1520$    & $1237$      & —          & —          & —           \\
\midrule
\multirow{4}{*}{RR} 
    & Obj Mean         & $-5.179$  & $-16.17$  & $-21.92$    & $-46.73$   & $-106.5$   & $-213.3$    \\
    & Obj Median       & $-5.217$  & $-16.21$  & $-21.89$    & $-46.76$   & $-106.5$   & $-213.3$    \\
    & Feasible      & $0\%$     & $0\%$      & $0\%$      & $0\%$      & $0\%$      & $0\%$       \\
    & Solved        & $100\%$   & $100\%$    & $100\%$    & $100\%$          & $100\%$          & $100\%$           \\
    & Time       & $0.417$   & $0.440$    & $0.583$    & $0.846$    & $2.639$    & $8.874$     \\
\midrule
\multirow{4}{*}{N1} 
    & Obj Mean         & $9.8e7$   & $1.7e17$   & $1.5e18$   & —          & —          & —           \\
    & Obj Med       & $9.600$   & $2.4e17$   & $1.4e18$   & —          & —          & —           \\
    & Feasible      & $100\%$   & $100\%$    & $100\%$    & —          & —          & —           \\
    & Solved        & $100\%$   & $100\%$    & $100\%$    & $0\%$          & $0\%$           & $0\%$            \\
    & Time       & $0.415$   & $0.498$    & $104.2$    & —          & —          & —           \\
\midrule[0.3pt]\bottomrule[1pt]
\end{tabular}
}
\end{table}

\begin{table}[ht!]
\centering
\caption{
\textbf{Results for INPs.}  
Each problem size is evaluated on a test set of 100 instances. ``Obj Mean'' and ``Obj Med'' represent the mean and median objective values for this minimization problem, with smaller values being better. ``Feasible'' denotes the fraction of feasible solutions, ``Solved'' denotes the percentage of instances where a solution (feasible or infeasible) was found within the time limit, and ``Time'' is the average solving/inference time per instance. ``—'' indicates that no solution was found within 1000 seconds. For methods achieving 100\% feasibility, we highlight in bold the best-performing metrics.}
\label{tab:NC_Stat}
\resizebox{0.8\linewidth}{!}{
\begin{tabular}{ll|rrrrrr}
\toprule[1pt]\midrule[0.3pt]
\textbf{} & \textbf{Metric} & \textbf{20×20} & \textbf{50×50} & \textbf{100×100} & \textbf{200×200} & \textbf{500×500}  & \textbf{1000×1000} \\
\midrule
\multirow{4}{*}{RC} 
    & Obj Mean       & $0.228$  & $0.771$  & $1.664$  & $1.472$  & $0.526$   & $1.422$  \\
    & Obj Med        & $0.217$  & $0.752$  & $1.594$  & $1.436$  & $0.526$   & $0.809$  \\
    & Feasible       & $100\%$  & $98\%$   & $100\%$  & $99\%$   & $96\%$    & $97\%$   \\
    & Time           & $\bm{0.0019}$ & $0.0020$ & $0.0022$ & $0.0022$ & $0.0029$  & $0.0040$ \\
\midrule
\multirow{4}{*}{RC-P} 
    & Obj Mean       & $0.228$  & $0.772$  & $1.664$  & $1.471$  & $0.524$   & $1.423$  \\
    & Obj Median     & $0.217$  & $0.752$  & $1.594$  & $1.436$  & $0.526$   & $0.809$  \\
    & Feasible       & $100\%$  & $100\%$  & $100\%$  & $100\%$  & $100\%$   & $100\%$  \\
    & Time           & $0.0045$ & $0.0058$ & $0.0060$ & $0.0054$ & $\bm{0.0061}$  & $\bm{0.0115}$ \\
\midrule
\multirow{4}{*}{LT} 
    & Obj Mean       & $0.195$  & $0.580$  & $0.669$  & $\bm{-0.356}$ & $-1.374$  & $-3.744$ \\
    & Obj Med        & $0.175$  & $0.566$  & $0.649$  & $\bm{-0.373}$ & $-1.594$  & $-3.716$ \\
    & Feasible       & $99\%$   & $98\%$   & $96\%$   & $100\%$  & $98\%$    & $99\%$   \\
    & Time           & $0.0019$ & $0.0020$ & $0.0021$ & $\bm{0.0023}$ & $0.0029$  & $0.0050$ \\
\midrule
\multirow{4}{*}{LT-P} 
    & Obj Mean       & $0.195$  & $\bm{0.580}$  & $\bm{0.669}$  & $\bm{-0.356}$ & $\bm{-1.374}$  & $\bm{-3.744}$ \\
    & Obj Median     & $0.175$  & $0.566$  & $\bm{0.649}$  & $\bm{-0.373}$ & $\bm{-1.594}$  & $\bm{-3.716}$ \\
    & Feasible       & $100\%$  & $100\%$  & $100\%$  & $100\%$  & $100\%$   & $100\%$  \\
    & Time           & $0.0048$ & $\bm{0.0050}$ & $\bm{0.0058}$ & $0.0056$ & $0.0072$  & $0.0117$ \\
\midrule
\multirow{5}{*}{EX} 
    & Obj Mean       & $\bm{-0.453}$ & $1.649$  & $256.93$ & —        & —         & —        \\
    & Obj Med        & $\bm{-0.463}$ & $\bm{-0.052}$ & $134.62$ & —        & —         & —        \\
    & Feasible       & $100\%$  & $100\%$  & $14\%$   & —        & —         & —        \\
    & Solved         & $100\%$  & $100\%$  & $14\%$   & $0\%$    & $0\%$     & $0\%$    \\
    & Time           & $0.9949$ & $1001$   & $1001$   & —        & —         & —        \\
\midrule
\multirow{5}{*}{RR} 
    & Obj Mean       & $-0.464$ & $-1.039$ & $-2.068$ & $-3.990$ & $-9.391$  & —         \\
    & Obj Med        & $-0.476$ & $-1.215$ & $-2.307$ & $-4.327$ & $-9.221$  & —         \\
    & Feasible       & $3\%$    & $0\%$    & $0\%$    & $0\%$    & $0\%$     & —         \\
    & Solved         & $100\%$  & $100\%$  & $100\%$  & $100\%$  & $100\%$   & $0\%$     \\
    & Time           & $0.996$  & $1.189$  & $4.600$  & $54.01$  & $449.0$   & —         \\
\midrule
\multirow{5}{*}{N1} 
    & Obj Mean       & $2.1e4$  & $3.7e6$  & $4411$   & —        & —         & —         \\
    & Obj Med        & $2.222$  & $45.85$  & $155.2$  & —        & —         & —         \\
    & Feasible       & $100\%$  & $100\%$  & $14\%$   & —        & —         & —         \\
    & Solved         & $100\%$  & $100\%$  & $14\%$   & $0\%$    & $0\%$     & $0\%$     \\
    & Time           & $0.144$  & $8.968$  & $940.4$  & —        & —         & —         \\
\midrule[0.3pt]\bottomrule[1pt]
\end{tabular}
}
\end{table}

\begin{table}[ht!]
\centering
\caption{
\textbf{Results for MIRBs.}
Each problem size is evaluated on a test set of 100 instances.
``Obj Mean'' and ``Obj Med'' represent the mean and median objective values for this minimization problem.
``Feasible'' denotes the fraction of feasible solutions, ``Solved'' denotes the percentage of instances where a solution was found within the time limit, and ``Time'' is the average solving/inference time per instance.
For methods achieving 100\% feasibility, we highlight in bold the best-performing metrics.
}
\label{tab:RB_Stat}
\resizebox{0.65\linewidth}{!}{
\begin{tabular}{ll|rrrrr}
\toprule[1pt]\midrule[0.3pt]
\textbf{} & \textbf{Metric}
& \textbf{2×4} & \textbf{20×4} & \textbf{200×4} & \textbf{2000×4} & \textbf{20000×4} \\
\midrule

\multirow{4}{*}{RC}
    & Obj Mean     & $23.27$  & $\bm{59.39}$  & $503.5$   & $5938$    & $6.7e4$ \\
    & Obj Med      & $21.48$  & $\bm{48.86}$  & $461.7$   & $5792$    & $6.7e4$ \\
    & Feasible     & $97\%$   & $100\%$       & $99\%$    & $99\%$    & $76\%$  \\
    & Time         & $0.0019$ & $\bm{0.0019}$ & $0.0021$  & $0.0033$  & $0.0121$ \\
\midrule

\multirow{4}{*}{RC-P}
    & Obj Mean     & $23.50$  & $\bm{59.39}$ & $\bm{504.2}$ & $5942$    & $9.8e4$ \\
    & Obj Med      & $21.48$  & $\bm{48.86}$ & $\bm{461.7}$ & $5792$    & $7.3e4$ \\
    & Feasible     & $100\%$  & $100\%$      & $100\%$      & $100\%$   & $100\%$ \\
    & Time         & $\bm{0.0062}$ & $0.0048$ & $\bm{0.0052}$ & $\bm{0.0070}$ & $0.0824$ \\
\midrule

\multirow{4}{*}{LT}
    & Obj Mean     & $23.18$  & $62.51$    & $622.8$   & $5612$    & $4.8e4$ \\
    & Obj Med      & $20.80$  & $63.40$    & $626.0$   & $5558$    & $3.5e4$ \\
    & Feasible     & $98\%$   & $100\%$    & $100\%$   & $97\%$    & $66\%$  \\
    & Time         & $0.0019$ & $0.0020$   & $0.0026$  & $0.0030$  & $0.0127$ \\
\midrule

\multirow{4}{*}{LT-P}
    & Obj Mean     & $23.33$  & $62.51$    & $622.8$   & $\bm{5615}$    & $\bm{8.0e4}$ \\
    & Obj Med      & $20.80$  & $63.40$    & $626.0$   & $\bm{5558}$    & $\bm{4.5e4}$ \\
    & Feasible     & $100\%$  & $100\%$    & $100\%$   & $100\%$   & $100\%$ \\
    & Time         & $\bm{0.0062}$ & $0.0055$ & $0.0062$  & $0.0071$ & $\bm{0.0639}$ \\
\midrule

\multirow{5}{*}{EX}
    & Obj Mean     & $\bm{19.62}$ & $64.67$  & $8.4e5$  & $4.7e10$ & $1.1e15$ \\
    & Obj Med      & $\bm{18.20}$ & $59.16$  & $908.8$  & $9262$   & $1.0e5$ \\
    & Feasible     & $100\%$       & $100\%$  & $100\%$  & $96\%$   & $78\%$ \\
    & Solved       & $100\%$       & $100\%$  & $100\%$  & $96\%$   & $78\%$ \\
    & Time         & $3.5090$      & $1005$   & $1002$   & $1002$   & $1040$ \\
\midrule

\multirow{5}{*}{RR}
    & Obj Mean     & $22.24$  & $1.2e4$  & $1.4e4$  & $2.1e6$  & $1.7e8$ \\
    & Obj Med      & $22.19$  & $51.17$  & $501.9$  & $5437$   & $7.0e6$ \\
    & Feasible     & $55\%$   & $59\%$   & $40\%$   & $6\%$    & $18\%$ \\
    & Solved       & $100\%$  & $100\%$  & $58\%$   & $7\%$    & $22\%$ \\
    & Time         & $0.1805$ & $0.5570$ & $1.2396$ & $9.2334$ & $1064$ \\
\midrule

\multirow{5}{*}{N1}
    & Obj Mean     & $40.37$  & $87.83$  & $3.7e8$  & $8.3e12$ & $1.2e15$ \\
    & Obj Med      & $27.93$  & $77.34$  & $957.4$  & $9379$   & $1.0e5$ \\
    & Feasible     & $100\%$  & $100\%$  & $100\%$  & $95\%$   & $78\%$ \\
    & Solved       & $100\%$  & $100\%$  & $100\%$  & $95\%$   & $78\%$ \\
    & Time         & $0.0323$ & $0.0813$ & $0.2608$ & $71.91$  & $782.1$ \\
\midrule[0.3pt]\bottomrule[1pt]

\end{tabular}
}
\end{table}


\paragraph{Q1. How do learning-based methods compare to traditional solvers and heuristics?}
As shown in~\Cref{tab:QP_Stat},~\Cref{tab:NC_Stat} and~~\Cref{tab:RB_Stat}, traditional methods (EX, RR, and N1) scale poorly on larger instances: EX and N1 often fail to find a feasible solution within the 1000-second time limit, N1 suffers from numerical instability, and RR frequently produces infeasible solutions as problem size increases.. In contrast, RC and LT consistently return high-quality solutions within milliseconds. For IQPs and INPs, they match or surpass the objective values obtained by EX while achieving substantially higher feasibility rates than heuristic baselines. For MIRBs, RC and LT often reach solution quality comparable to or better than EX. Overall, the learning-based methods provide competitive solution quality with several orders of magnitude speedups.

\paragraph{Q2. How effective is the integer feasibility projection?}
\Cref{tab:QP_Stat}, \Cref{tab:NC_Stat}, and \Cref{tab:RB_Stat} show that RC-P and LT-P successfully recover feasibility on all test instances with small computational overhead. For IQPs and INPs, constraint violations produced by RC and LT are sparse and of small magnitude, allowing the projection step to restore feasibility with negligible impact on objective quality. In contrast, for MIRBs, constraint violations become more pronounced as the problem dimension increases, making projection essential: while feasibility degrades for RC and LT at larger scales, RC-P and LT-P satisfy all constraints even for the largest instances. These trends are further illustrated by the violation visualizations in Appendix~\ref{app:viol_figures}.

\paragraph{Q3. How does the choice of penalty weight affect performance?}
The penalty weight $\lambda$ in~\Cref{eq:loss} balances objective minimization against constraint satisfaction. To evaluate its impact, we vary $\lambda$ from $0.1$ to $1000$ on $1000\times1000$ INPs using RC, LT, and their projection-enhanced variants RC-P and LT-P. As shown in~\Cref{fig:penalty}, smaller penalty values lead to better objective values but also increase the proportion of infeasible solutions for RC and LT. Larger penalty weights improve feasibility but may degrade objective quality. Notably, applying the projection step (up to $1000$ iterations) restores feasibility even when the penalty is too small for RC and LT to satisfy the constraints, while largely preserving the low objective values associated with small penalties. This pattern is consistent across all benchmarks and suggests that RC-P and LT-P can effectively operate with penalty weights smaller than those used in the main experiments.

\begin{figure}[htbp!]
    \centering    \includegraphics[width=0.8\textwidth]{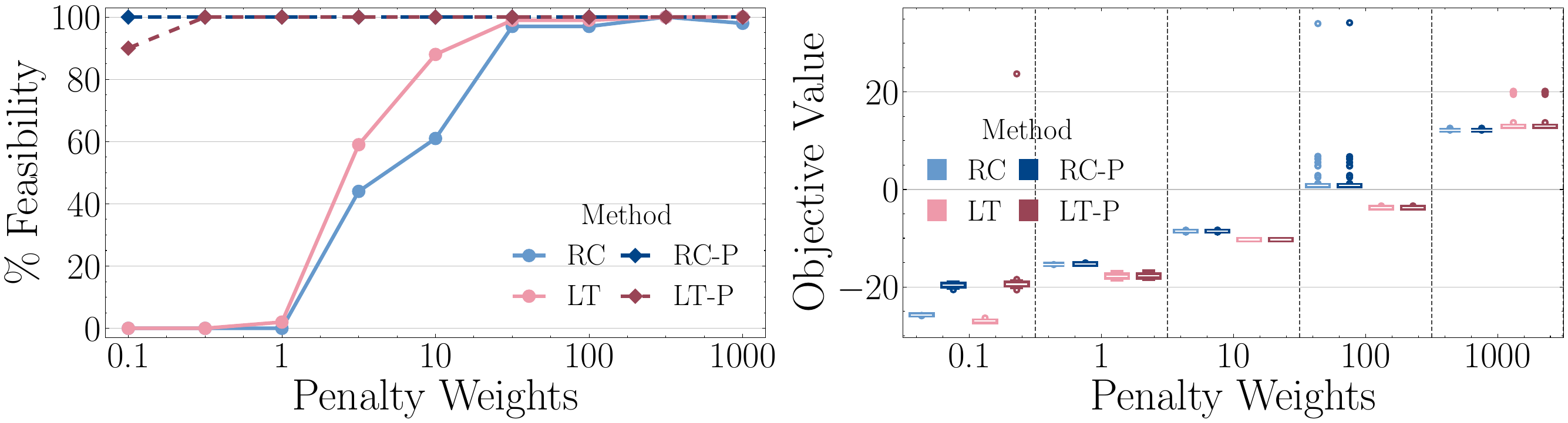}
    \caption{Illustration of the proportion of feasible solutions (Top) and objective value (Bottom) for $1000\times1000$ INC on the test set. As the penalty weight increases, the fraction of feasible solutions increases while the objective value generally deteriorates.}
    \label{fig:penalty}
\end{figure}

\paragraph{Q4. How long is the training time?}
In addition to evaluating solution quality, feasibility, and inference speed, we measure the offline training time across different problem sizes. As shown in~\Cref{tab:train_times_all}, the learning-based methods exhibit favorable scaling: training requires only a few minutes for smaller instances and remains under 30 minutes even for the largest problems. For large-scale settings, a single training run is often faster than finding the first feasible solution for one instance using the exact solver. When solving many instances repeatedly, this offline cost is amortized, making the proposed approach particularly attractive in applications requiring rapid or large-scale deployment.

\begin{table}[htb!]
\centering
\caption{Training Times (in seconds) for RC and LT methods across different problem sizes. Each method was set to train 9{,}000 instances for each problem for 200 epochs, with 1{,}000 instances reserved for validation per epoch and early stopping applied.}
\label{tab:train_times_all}

\resizebox{0.8\linewidth}{!}{%
\begin{minipage}{\linewidth}

\centering
\caption*{\textbf{Training Times for IQPs}}
\label{tab:QP_Time}
\begin{tabular}{l|rrrrrr}
\toprule[1pt]\midrule[0.3pt]
\textbf{Method} & \textbf{20×20} & \textbf{50×50} & \textbf{100×100} & \textbf{200×200} & \textbf{500×500} & \textbf{1000×1000}\\
\midrule
RC & 153.98 & 237.11 & 141.15 & 149.43 & 606.23 & 727.32\\
LT & 154.33 & 158.61 & 128.86 & 139.17 & 458.62 & 462.41\\
\midrule[0.3pt]\bottomrule[1pt]
\end{tabular}

\vspace{0.8em}

\centering
\caption*{\textbf{Training Times for INPs}}
\label{tab:NC_Time}
\begin{tabular}{l|rrrrrr}
\toprule[1pt]\midrule[0.3pt]
\textbf{Method} & \textbf{20×20} & \textbf{50×50} & \textbf{100×100} & \textbf{200×200} & \textbf{500×500} & \textbf{1000×1000}\\
\midrule
RC & 173.02 & 138.53 & 136.01 & 104.05 & 116.01 & 156.85 \\
LT & 104.35 & 88.41 & 111.38 & 89.24 & 230.52 & 195.67\\
\midrule[0.3pt]\bottomrule[1pt]
\end{tabular}

\vspace{0.8em}

\centering
\caption*{\textbf{Training Times for MIRBs}}
\label{tab:RB_Time}
\begin{tabular}{l|rrrrr}
\toprule[1pt]\midrule[0.3pt]
\textbf{Method} & \textbf{2×4} & \textbf{20×4} & \textbf{200×4} & \textbf{2000×4} & \textbf{20000×4} \\
\midrule
RC & 230.68 & 112.35 & 75.49 & 106.76 & 5227.05 \\
LT & 126.60 & 125.11 & 86.43 & 84.61 & 6508.41 \\
\midrule[0.3pt]\bottomrule[1pt]
\end{tabular}

\end{minipage}%
} 

\end{table}

\paragraph{Q5. How do the proposed correction layers contribute to solution quality and feasibility?}
To isolate the contribution of the correction layers $\varphi_{\Theta_2}$, we compare the full models with two reduced variants. As shown in~\Cref{tab:Abla_all}, both Rounding after Learning (RL) and Rounding with STE (RS) underperform the proposed RC and LT methods. In RL, only the predictor $\pi_{\Theta_1}$ is trained and integer rounding is applied post-hoc, so the model is never optimized on the resulting mixed-integer solution. This leads to large objective deviations and frequent constraint violations. RS incorporates integer outputs during training via the straight-through estimator (\Cref{algo:ste}), improving feasibility relative to RL, but the rounding rule is fixed and instance-agnostic. As a result, RS cannot adapt rounding decisions to the relaxed solution or constraint structure, and consequently lacks the refinement achieved by the learnable correction layers $\varphi_{\Theta_2}$. 

\begin{table}[htb!]
\centering
\caption{Ablations for IQPs, INPs, and MIRB. Each problem size is evaluated on a test set of 100 instances. ``Obj Mean" and ``Obj Med" represent the mean and median objective values for this minimization problem, with smaller values being better. ``Feasible" denotes the fraction of feasible solutions, and ``Time     "  is the average solving/inference time per instance.}
\label{tab:Abla_all}

\resizebox{0.9\linewidth}{!}{%
\begin{minipage}{\linewidth}

\centering
\caption*{\textbf{Ablation Study for IQPs.}}
\label{tab:QP_Abla}
\begin{tabular}{ll|rrrrrr}
\toprule[1pt]\midrule[0.3pt]
\textbf{Method} & \textbf{Metric} & \textbf{20×20} & \textbf{50×50} & \textbf{100×100} & \textbf{200×200} & \textbf{500×500} & \textbf{1000×1000} \\
\midrule
\multirow{4}{*}{RL} 
    & Obj Mean         & $-4.726$  & $-14.52$   & $-17.22$   & $-37.14$  & $-89.81$   & $-176.6$  \\
    & Obj Median       & $-4.716$  & $-14.52$   & $-17.27$   & $-37.15$  & $-89.81$   & $-176.6$  \\
    & Feasible         & $64\%$    & $42\%$     & $23\%$     & $10\%$    & $0\%$      & $0\%$     \\
    & Time (Sec)       & $0.0004$  & $0.0004$   & $0.0005$   & $0.0005$  & $0.0005$   & $0.0011$  \\
\midrule
\multirow{4}{*}{RS} 
    & Obj Mean         & $-3.929$  & $-11.93$   & $-10.58$   & $-24.72$  & $-54.93$   & $-110.7$  \\
    & Obj Median       & $-3.963$  & $-11.96$   & $-10.58$   & $-24.72$  & $-54.93$   & $-110.6$  \\
    & Feasible         & $100\%$   & $100\%$    & $100\%$    & $100\%$   & $100\%$    & $100\%$   \\
    & Time (Sec)       & $0.0010$  & $0.0011$   & $0.0013$   & $0.0012$  & $0.0016$   & $0.0031$  \\
\midrule[0.3pt]\bottomrule[1pt]
\end{tabular}

\vspace{0.8em}

\centering
\caption*{\textbf{Ablation Study for INPs.}}
\label{tab:NC_Abla}
\begin{tabular}{ll|rrrrrr}
\toprule[1pt]\midrule[0.3pt]
\textbf{Method} & \textbf{Metric} & \textbf{20×20} & \textbf{50×50} & \textbf{100×100} & \textbf{200×200} & \textbf{500×500} & \textbf{1000×1000} \\
\midrule
\multirow{4}{*}{RL} 
    & Obj Mean       & $-0.138$  & $-0.629$  & $-1.581$  & $-4.196$  & $-11.531$  & $-23.64$  \\
    & Obj Median     & $-0.148$  & $-0.655$  & $-1.554$  & $-4.196$  & $-11.531$  & $-23.64$  \\
    & Feasible       & $87\%$    & $51\%$    & $15\%$    & $0\%$     & $0\%$      & $0\%$     \\
    & Time (Sec)     & $0.0005$  & $0.0005$  & $0.0006$  & $0.0006$  & $0.0006$   & $0.0013$  \\
\midrule
\multirow{4}{*}{RS} 
    & Obj Mean       & $0.292$   & $1.734$   & $2.849$   & $4.921$   & $9.511$    & $25.36$   \\
    & Obj Median     & $0.284$   & $1.736$   & $2.841$   & $4.907$   & $9.511$    & $25.36$   \\
    & Feasible       & $100\%$   & $100\%$   & $100\%$   & $100\%$   & $100\%$    & $100\%$   \\
    & Time (Sec)     & $0.0012$  & $0.0011$  & $0.0012$  & $0.0013$  & $0.0018$   & $0.0031$  \\
\midrule[0.3pt]\bottomrule[1pt]
\end{tabular}

\vspace{0.8em}

\centering
\caption*{\textbf{Ablation Study for MIRBs.}}
\label{tab:RB_Abla}
\begin{tabular}{ll|rrrrr}
\toprule[1pt]\midrule[0.3pt]
\textbf{Method} & \textbf{Metric} & \textbf{2×4} & \textbf{20×4} & \textbf{200×4} & \textbf{2000×4} & \textbf{20000×4} \\
\midrule
\multirow{4}{*}{RL} 
    & Obj Mean       & $58.34$   & $63.70$  & $605.9$  & $6222$   & $68364$  \\
    & Obj Median     & $58.00$   & $61.95$  & $609.0$  & $5950$   & $69087$  \\
    & Feasible       & $14\%$    & $64\%$   & $56\%$   & $72\%$   & $69\%$   \\
    & Time (Sec)     & $0.0006$  & $0.0005$ & $0.0005$ & $0.0008$ & $0.0014$ \\
\midrule
\multirow{4}{*}{RS} 
    & Obj Mean       & $25.095$  & $69.36$  & $684.7$  & $6852$   & $72910$ \\
    & Obj Median     & $25.353$  & $68.58$  & $663.1$  & $6509$   & $68904$ \\
    & Feasible       & $100\%$   & $97\%$   & $100\%$  & $99\%$   & $61\%$  \\
    & Time (Sec)     & $0.0010$  & $0.0010$ & $0.0012$ & $0.0019$ & $0.0103$ \\
\midrule[0.3pt]\bottomrule[1pt]
\end{tabular}

\end{minipage}%
} 

\end{table}

\section{Conclusion}
\label{sec:cncl}

We propose a fully learning-based, solver-free framework for pMINLP, enabling neural networks to generate feasible and high-quality mixed-integer solutions without optimal labels. The method integrates differentiable integer correction layers trained with a self-supervised objective, and augments them with an efficient gradient-based feasibility projection that guarantees satisfaction of mixed-integer inequality constraints under mild structural assumptions. The projection adds negligible inference-time overhead, making the approach practical for large-scale problems.
Empirically, the proposed methods outperform classical heuristics and match or surpass exact solvers across diverse benchmark classes, including very large-scale instances on which exact approaches become intractable. To the best of our knowledge, this is the first learning-to-optimize framework to provide feasibility guarantees for general parametric MINLPs that can successfully solve problems with tens of thousands of decision variables.

Several limitations motivate future research. Our feasibility guarantees currently apply only to inequality-constrained MINLPs and require specific yet generic assumptions; extending these guarantees to equality-constrained MINLPs remains an open problem. Moreover, hybrid architectures that combine differentiable optimization layers~\citep{DiffCVxLayers2019}, variable completion~\citep{donti2021dc3}, or domain-specific architectures such as \citet{pan2020deepopf, tordesillas2023rayen} may further enhance feasibility and generalization.

Taken together, our results suggest that learning-based approaches can serve as a scalable and practical alternative to exact solvers for a wide range of parametric MINLPs, opening new possibilities for deploying optimization-driven intelligence in large-scale, real-time, and data-rich scientific and engineering systems.

\section*{Acknowledgments}

This research was supported by the Ralph O’Connor Sustainable Energy Institute at Johns Hopkins University.
Early version of this research was also supported by the Data Model Convergence (DMC) initiative via the Laboratory Directed Research and Development (LDRD) investments at Pacific Northwest National Laboratory (PNNL). PNNL is a multi-program national laboratory operated for the U.S. Department of Energy (DOE) by Battelle Memorial Institute under Contract No. DE-AC05-76RL0-1830.



\newpage
\bibliographystyle{unsrtnat}
\bibliography{arxiv_ref}

@article{Amos_amortized2022,
  author       = {Brandon Amos},
  title        = {Tutorial on amortized optimization for learning to optimize over continuous
                  domains},
  journal      = {CoRR},
  volume       = {abs/2202.00665},
  year         = {2022},
  url          = {https://arxiv.org/abs/2202.00665},
  eprinttype    = {arXiv},
  eprint       = {2202.00665},
  bibsource    = {dblp computer science bibliography, https://dblp.org}
}

@article{bolte2007_clarke,
  title={Clarke subgradients of stratifiable functions},
  author={Bolte, Jérôme and Daniilidis, Aris and Lewis, Adrian S and Shiota, Masahiro},
  journal={SIAM Journal on Optimization},
  volume={18},
  number={2},
  pages={556--572},
  year={2007},
  publisher={SIAM}
}

@inproceedings{jin2017escape,
  title={How to escape saddle points efficiently},
  author={Jin, Chi and Ge, Rong and Netrapalli, Praneeth and Kakade, Sham and Jordan, Michael I},
  booktitle={Proceedings of the 34th International Conference on Machine Learning (ICML)},
  volume={70},
  pages={1724--1732},
  year={2017},
  organization={PMLR}
}

@article{attouch2013_convergence,
  title={Convergence of descent methods for semi-algebraic and tame problems: proximal algorithms, forward--backward splitting, and regularized Gauss--Seidel methods},
  author={Attouch, Hedy and Bolte, Jérôme and Svaiter, Benar Fux},
  journal={Mathematical Programming},
  volume={137},
  pages={91--129},
  year={2013},
  publisher={Springer}
}

@article{DiffCVxLayers2019,
  title={Differentiable Convex Optimization Layers},
  author={Akshay Agrawal and Brandon Amos and Shane Barratt and Stephen Boyd and Steven Diamond and Zico Kolter},
  journal={ArXiv},
  year={2019},
  volume={abs/1910.12430}
}

@inproceedings{pathak2015constrained,
  title={Constrained convolutional neural networks for weakly supervised segmentation},
  author={Pathak, Deepak and Krahenbuhl, Philipp and Darrell, Trevor},
  booktitle={Proceedings of the IEEE international conference on computer vision},
  pages={1796--1804},
  year={2015}
}

@article{hendriks2020linearly,
  title={Linearly constrained neural networks},
  author={Hendriks, Johannes and Jidling, Carl and Wills, Adrian and Sch{\"o}n, Thomas},
  journal={arXiv:2002.01600},
  year={2020}
}

@article{jia2017constrained,
  title={Constrained deep weak supervision for histopathology image segmentation},
  author={Jia, Zhipeng and Huang, Xingyi and Eric, I and Chang, Chao and Xu, Yan},
  journal={IEEE transactions on medical imaging},
  volume={36},
  number={11},
  pages={2376--2388},
  year={2017},
  publisher={IEEE}
}

@article{marquez2017imposing,
  title={Imposing hard constraints on deep networks: Promises and limitations},
  author={M{\'a}rquez-Neila, Pablo and Salzmann, Mathieu and Fua, Pascal},
  journal={arXiv preprint arXiv:1706.02025},
  year={2017}
}

@inproceedings{kervadec2022constrained,
  title={Constrained deep networks: Lagrangian optimization via log-barrier extensions},
  author={Kervadec, Hoel and Dolz, Jose and Yuan, Jing and Desrosiers, Christian and Granger, Eric and Ayed, Ismail Ben},
  booktitle={2022 30th European Signal Processing Conference (EUSIPCO)},
  year={2022},
}

@article{kingma2014adam,
  title={Adam: A method for stochastic optimization},
  author={Kingma, Diederik P and Ba, Jimmy},
  journal={arXiv preprint arXiv:1412.6980},
  year={2014}
}

@article{nair2020solving,
  title={Solving mixed integer programs using neural networks},
  author={Nair, Vinod and Bartunov, Sergey and Gimeno, Felix and von Glehn, Ingrid and Lichocki, Pawel and Lobov, Ivan and O'Donoghue, Brendan and Sonnerat, Nicolas and Tjandraatmadja, Christian and Wang, Pengming and others},
  journal={arXiv preprint arXiv:2012.13349},
  year={2020}
}

@article{gleixner2021miplib,
  title={MIPLIB 2017: data-driven compilation of the 6th mixed-integer programming library},
  author={Gleixner, Ambros and Hendel, Gregor and Gamrath, Gerald and Achterberg, Tobias and Bastubbe, Michael and Berthold, Timo and Christophel, Philipp and Jarck, Kati and Koch, Thorsten and Linderoth, Jeff and others},
  journal={Mathematical Programming Computation},
  volume={13},
  number={3},
  pages={443--490},
  year={2021},
  publisher={Springer}
}

@INPROCEEDINGS{Schouwenaars2001,
  title={Mixed integer programming for multi-vehicle path planning},
  author={Schouwenaars, Tom and De Moor, Bart and Feron, Eric and How, Jonathan},
  booktitle={2001 European control conference (ECC)},
  pages={2603--2608},
  year={2001},
  organization={IEEE}
}

@ARTICLE{Marcucci2021,
  title={Warm start of mixed-integer programs for model predictive control of hybrid systems},
  author={Marcucci, Tobia and Tedrake, Russ},
  journal={IEEE Transactions on Automatic Control},
  volume={66},
  number={6},
  pages={2433--2448},
  year={2020},
  publisher={IEEE}
}

@article{nazir2021guaranteeing,
  title={Guaranteeing a physically realizable battery dispatch without charge-discharge complementarity constraints},
  author={Nazir, Nawaf and Almassalkhi, Mads},
  journal={IEEE Transactions on Smart Grid},
  volume={14},
  number={3},
  pages={2473--2476},
  year={2021},
  publisher={IEEE}
}

@article{fischetti2005feasibility,
  title={The feasibility pump},
  author={Fischetti, Matteo and Glover, Fred and Lodi, Andrea},
  journal={Mathematical Programming},
  volume={104},
  pages={91--104},
  year={2005},
  publisher={Springer}
}

@article{berthold2014rens,
  title={RENS: the optimal rounding},
  author={Berthold, Timo},
  journal={Mathematical Programming Computation},
  volume={6},
  pages={33--54},
  year={2014},
  publisher={Springer}
}

@inproceedings{ding2020accelerating,
  title={Accelerating Primal Solution Findings for Mixed Integer Programs Based on Solution Prediction},
  author={Ding, Jian-Ya and Zhang, Chao and Shen, Lei and Li, Shengyin and Wang, Bing and Xu, Yinghui and Song, Le},
  booktitle={Proceedings of the AAAI Conference on Artificial Intelligence},
  year={2020}
}

@inproceedings{ferber2023surco,
  title={Surco: Learning linear surrogates for combinatorial nonlinear optimization problems},
  author={Ferber, Aaron M and Huang, Taoan and Zha, Daochen and Schubert, Martin and Steiner, Benoit and Dilkina, Bistra and Tian, Yuandong},
  booktitle={International Conference on Machine Learning},
  year={2023},
  organization={PMLR}
}

@inproceedings{gregor2010learning,
  title={Learning fast approximations of sparse coding},
  author={Gregor, Karol and LeCun, Yann},
  booktitle={Proceedings of the 27th international conference on international conference on machine learning},
  pages={399--406},
  year={2010}
}

@inproceedings{donti2021dc3,
  title={{DC3: A learning method for optimization with hard constraints}},
  author={Donti, Priya and Rolnick, David and Kolter, J Zico},
  booktitle={International Conference on Learning Representations},
  journal={arXiv preprint arXiv:2104.12225},
  year={2021}
}

@article{jang2016categorical,
  title={Categorical reparameterization with gumbel-softmax},
  author={Jang, Eric and Gu, Shixiang and Poole, Ben},
  journal={arXiv preprint arXiv:1611.01144},
  year={2016}
}

@article{bengio2013estimating,
  title={Estimating or propagating gradients through stochastic neurons for conditional computation},
  author={Bengio, Yoshua and L{\'e}onard, Nicholas and Courville, Aaron},
  journal={arXiv preprint arXiv:1308.3432},
  year={2013}
}

@misc{Neuromancer2023,
title={NeuroMANCER: Neural Modules with Adaptive Nonlinear Constraints and Efficient Regularizations},
author={Drgona, Jan and Tuor, Aaron and Koch, James and Shapiro, Madelyn and Jacob, Bruno and Vrabie, Draguna},
url= {https://github.com/pnnl/neuromancer}, 
year={2023}
}

@misc{gurobi,
  author = {{Gurobi Optimization, LLC}},
  title = {{Gurobi Optimizer Reference Manual}},
  year = 2021,
  url = "https://www.gurobi.com"
}

@article{KLEINERT2021100007,
  title={A survey on mixed-integer programming techniques in bilevel optimization},
  author={Kleinert, Thomas and Labb{\'e}, Martine and Ljubi{\'c}, Ivana and Schmidt, Martin},
  journal={EURO Journal on Computational Optimization},
  volume={9},
  pages={100007},
  year={2021},
  publisher={Elsevier}
}

@inproceedings{fioretto2020predicting,
  title={Predicting ac optimal power flows: Combining deep learning and lagrangian dual methods},
  author={Fioretto, Ferdinando and Mak, Terrence WK and Van Hentenryck, Pascal},
  booktitle={Proceedings of the AAAI conference on artificial intelligence},
  year={2020}
}

@article{pan2020deepopf,
  title={Deepopf: A deep neural network approach for security-constrained dc optimal power flow},
  author={Pan, Xiang and Zhao, Tianyu and Chen, Minghua and Zhang, Shengyu},
  journal={IEEE Transactions on Power Systems},
  volume={36},
  number={3},
  pages={1725--1735},
  year={2020},
  publisher={IEEE}
}

@article{vinyals2015pointer,
  title={Pointer networks},
  author={Vinyals, Oriol and Fortunato, Meire and Jaitly, Navdeep},
  journal={Advances in neural information processing systems},
  volume={28},
  year={2015}
}

@article{khalil2017learning,
  title={Learning combinatorial optimization algorithms over graphs},
  author={Dai, Hanjun and Khalil, Elias and Zhang, Yuyu and Dilkina, Bistra and Song, Le},
  journal={Advances in neural information processing systems},
  volume={30},
  year={2017}
}

@article{fletcher1994solving,
  title={Solving mixed integer nonlinear programs by outer approximation},
  author={Fletcher, Roger and Leyffer, Sven},
  journal={Mathematical programming},
  volume={66},
  pages={327--349},
  year={1994},
  publisher={Springer}
}

@article{belotti2009branching,
  title={Branching and bounds tightening techniques for non-convex {MINLP}},
  author={Belotti, Pietro and Lee, Jon and Liberti, Leo and Margot, Fran{\c{c}}ois and W{\"a}chter, Andreas},
  journal={Optimization Methods \& Software},
  volume={24},
  number={4-5},
  pages={597--634},
  year={2009},
  publisher={Taylor \& Francis}
}

@book{nowak2005relaxation,
  title={Relaxation and decomposition methods for mixed integer nonlinear programming},
  author={Nowak, Ivo},
  volume={152},
  year={2005},
  publisher={Springer Science \& Business Media}
}

@book{land2010automatic,
  title={An automatic method for solving discrete programming problems},
  author={Land, Ailsa H and Doig, Alison G},
  year={2010},
  publisher={Springer}
}

@article{crama2005local,
  title={Local search in combinatorial optimization},
  author={Crama, Yves and Kolen, Antoon WJ and Pesch, EJ},
  journal={Artificial Neural Networks: An Introduction to ANN Theory and Practice},
  pages={157--174},
  year={2005},
  publisher={Springer}
}

@article{johnson1997traveling,
  title={The traveling salesman problem: a case study},
  author={Johnson, David S and McGeoch, Lyle A},
  journal={Local search in combinatorial optimization},
  pages={215--310},
  year={1997},
  publisher={Chichester}
}

@inproceedings{khalil2016learning,
  title={Learning to branch in mixed integer programming},
  author={Khalil, Elias and Le Bodic, Pierre and Song, Le and Nemhauser, George and Dilkina, Bistra},
  booktitle={Proceedings of the AAAI Conference on Artificial Intelligence},
  year={2016}
}

@article{alvarez2017machine,
  title={A machine learning-based approximation of strong branching},
  author={Alvarez, Alejandro Marcos and Louveaux, Quentin and Wehenkel, Louis},
  journal={INFORMS Journal on Computing},
  volume={29},
  number={1},
  pages={185--195},
  year={2017},
  publisher={INFORMS}
}

@inproceedings{zarpellon2021parameterizing,
  title={Parameterizing branch-and-bound search trees to learn branching policies},
  author={Zarpellon, Giulia and Jo, Jason and Lodi, Andrea and Bengio, Yoshua},
  booktitle={Proceedings of the AAAI Conference on Artificial Intelligence},
  year={2021}
}

@article{GCNN,
  title={Exact combinatorial optimization with graph convolutional neural networks},
  author={Gasse, Maxime and Ch{\'e}telat, Didier and Ferroni, Nicola and Charlin, Laurent and Lodi, Andrea},
  journal={Advances in neural information processing systems},
  volume={32},
  year={2019}
}

@article{he2014learning,
  title={Learning to search in branch and bound algorithms},
  author={He, He and Daume III, Hal and Eisner, Jason M},
  journal={Advances in neural information processing systems},
  volume={27},
  year={2014}
}

@inproceedings{Deza_2023, 
   series={IJCAI-2023},
   title={Machine Learning for Cutting Planes in Integer Programming: A Survey},
   booktitle={Proceedings of the Thirty-Second International Joint Conference on Artificial Intelligence},
   author={Deza, Arnaud and Khalil, Elias B.},
   year={2023},
   month=aug, collection={IJCAI-2023} 
}

@article{song2020general,
  title={A general large neighborhood search framework for solving integer linear programs},
  author={Song, Jialin and Yue, Yisong and Dilkina, Bistra and others},
  journal={Advances in Neural Information Processing Systems},
  volume={33},
  pages={20012--20023},
  year={2020}
}

@article{bonami2022classifier,
  title={A classifier to decide on the linearization of mixed-integer quadratic problems in CPLEX},
  author={Bonami, Pierre and Lodi, Andrea and Zarpellon, Giulia},
  journal={Operations research},
  volume={70},
  number={6},
  pages={3303--3320},
  year={2022},
  publisher={INFORMS}
}

@misc{MIPcc23,
    title={The {MIP Workshop 2023} Computational Competition on Reoptimization}, 
    author={Suresh Bolusani and Mathieu Besançon and Ambros Gleixner and Timo Berthold and Claudia D'Ambrosio and Gonzalo Muñoz and Joseph Paat and Dimitri Thomopulos},
    year={2023},
    url={http://arxiv.org/abs/2311.14834}
}

@article{tordesillas2023rayen,
  title={Rayen: Imposition of hard convex constraints on neural networks},
  author={Tordesillas, Jesus and How, Jonathan P and Hutter, Marco},
  journal={arXiv preprint arXiv:2307.08336},
  year={2023}
}

@inproceedings{park2023self,
  title={Self-supervised primal-dual learning for constrained optimization},
  author={Park, Seonho and Van Hentenryck, Pascal},
  booktitle={Proceedings of the AAAI Conference on Artificial Intelligence},
  volume={37},
  pages={4052--4060},
  year={2023}
}

@article{raissi2019physics,
  title={Physics-informed neural networks: A deep learning framework for solving forward and inverse problems involving nonlinear partial differential equations},
  author={Raissi, Maziar and Perdikaris, Paris and Karniadakis, George E},
  journal={Journal of Computational physics},
  volume={378},
  pages={686--707},
  year={2019},
  publisher={Elsevier}
}

@article{maddison2016concrete,
  title={The concrete distribution: A continuous relaxation of discrete random variables},
  author={Maddison, Chris J and Mnih, Andriy and Teh, Yee Whye},
  journal={arXiv preprint arXiv:1611.00712},
  year={2016}
}

@article{li2024pdhg,
  title={PDHG-Unrolled Learning-to-Optimize Method for Large-Scale Linear Programming},
  author={Li, Bingheng and Yang, Linxin and Chen, Yupeng and Wang, Senmiao and Chen, Qian and Mao, Haitao and Ma, Yao and Wang, Akang and Ding, Tian and Tang, Jiliang and others},
  journal={arXiv preprint arXiv:2406.01908},
  year={2024}
}

@article{yang2024efficient,
  title={An Efficient Unsupervised Framework for Convex Quadratic Programs via Deep Unrolling},
  author={Yang, Linxin and Li, Bingheng and Ding, Tian and Wu, Jianghua and Wang, Akang and Wang, Yuyi and Tang, Jiliang and Sun, Ruoyu and Luo, Xiaodong},
  journal={arXiv preprint arXiv:2412.01051},
  year={2024}
}

@article{mattick2023reinforcement,
  title={Reinforcement learning for node selection in branch-and-bound},
  author={Mattick, Alexander and Mutschler, Christopher},
  journal={arXiv preprint arXiv:2310.00112},
  year={2023}
}

@article{wang2023learning,
  title={Learning cut selection for mixed-integer linear programming via hierarchical sequence model},
  author={Wang, Zhihai and Li, Xijun and Wang, Jie and Kuang, Yufei and Yuan, Mingxuan and Zeng, Jia and Zhang, Yongdong and Wu, Feng},
  journal={arXiv preprint arXiv:2302.00244},
  year={2023}
}

@article{xin2021neurolkh,
  title={Neurolkh: Combining deep learning model with lin-kernighan-helsgaun heuristic for solving the traveling salesman problem},
  author={Xin, Liang and Song, Wen and Cao, Zhiguang and Zhang, Jie},
  journal={Advances in Neural Information Processing Systems},
  volume={34},
  pages={7472--7483},
  year={2021}
}

@article{han2023gnn,
  title={A gnn-guided predict-and-search framework for mixed-integer linear programming},
  author={Han, Qingyu and Yang, Linxin and Chen, Qian and Zhou, Xiang and Zhang, Dong and Wang, Akang and Sun, Ruoyu and Luo, Xiaodong},
  journal={arXiv preprint arXiv:2302.05636},
  year={2023}
}

@article{joshi2019efficient,
  title={An efficient graph convolutional network technique for the travelling salesman problem},
  author={Joshi, Chaitanya K and Laurent, Thomas and Bresson, Xavier},
  journal={arXiv preprint arXiv:1906.01227},
  year={2019}
}

@article{bello2016neural,
  title={Neural combinatorial optimization with reinforcement learning},
  author={Bello, Irwan and Pham, Hieu and Le, Quoc V and Norouzi, Mohammad and Bengio, Samy},
  journal={arXiv preprint arXiv:1611.09940},
  year={2016}
}

@article{kool2018attention,
  title={Attention, learn to solve routing problems!},
  author={Kool, Wouter and Van Hoof, Herke and Welling, Max},
  journal={arXiv preprint arXiv:1803.08475},
  year={2018}
}

@article{vesselinova2020learning,
  title={Learning combinatorial optimization on graphs: A survey with applications to networking},
  author={Vesselinova, Natalia and Steinert, Rebecca and Perez-Ramirez, Daniel F and Boman, Magnus},
  journal={IEEE Access},
  volume={8},
  pages={120388--120416},
  year={2020},
  publisher={IEEE}
}

@article{luken2024self,
  title={Self-Supervised Learning of Iterative Solvers for Constrained Optimization},
  author={L{\"u}ken, Lukas and Lucia, Sergio},
  journal={arXiv preprint arXiv:2409.08066},
  year={2024}
}

@article{rangarajan2022expressing,
  title={Expressing linear equality constraints in feedforward neural networks},
  author={Rangarajan, Anand and He, Pan and Lee, Jaemoon and Banerjee, Tania and Ranka, Sanjay},
  journal={arXiv preprint arXiv:2211.04395},
  year={2022}
}

@inproceedings{frerix2020homogeneous,
  title={Homogeneous linear inequality constraints for neural network activations},
  author={Frerix, Thomas and Nie{\ss}ner, Matthias and Cremers, Daniel},
  booktitle={Proceedings of the IEEE/CVF Conference on Computer Vision and Pattern Recognition Workshops},
  pages={748--749},
  year={2020}
}

@article{sahinidis1996baron,
  title={BARON: A general purpose global optimization software package},
  author={Sahinidis, Nikolaos V},
  journal={Journal of global optimization},
  volume={8},
  number={2},
  pages={201--205},
  year={1996},
  publisher={Springer}
}

@article{misener2014antigone,
  title={ANTIGONE: algorithms for continuous/integer global optimization of nonlinear equations},
  author={Misener, Ruth and Floudas, Christodoulos A},
  journal={Journal of Global Optimization},
  volume={59},
  number={2},
  pages={503--526},
  year={2014},
  publisher={Springer}
}

@inproceedings{huangcontrastive,
  title={Contrastive predict-and-search for mixed integer linear programs},
  author={Huang, Taoan and Ferber, Aaron M and Zharmagambetov, Arman and Tian, Yuandong and Dilkina, Bistra},
  booktitle={Forty-first International Conference on Machine Learning},
  year={2024}
}

@article{courbariaux2015binaryconnect,
  title={Binaryconnect: Training deep neural networks with binary weights during propagations},
  author={Courbariaux, Matthieu and Bengio, Yoshua and David, Jean-Pierre},
  journal={Advances in neural information processing systems},
  volume={28},
  year={2015}
}

@article{hubara2016binarized,
  title={Binarized neural networks},
  author={Hubara, Itay and Courbariaux, Matthieu and Soudry, Daniel and El-Yaniv, Ran and Bengio, Yoshua},
  journal={Advances in neural information processing systems},
  volume={29},
  year={2016}
}

@book{vazirani2001approximation,
  title={Approximation algorithms},
  author={Vazirani, Vijay V},
  volume={1},
  year={2001},
  publisher={Springer}
}

@book{schrijver1998theory,
  title={Theory of linear and integer programming},
  author={Schrijver, Alexander},
  year={1998},
  publisher={John Wiley \& Sons}
}

@article{frank2016introduction,
  title={An introduction to optimal power flow: Theory, formulation, and examples},
  author={Frank, Stephen and Rebennack, Steffen},
  journal={IIE transactions},
  volume={48},
  number={12},
  pages={1172--1197},
  year={2016},
  publisher={Taylor \& Francis}
}

@article{hahn2017mixed,
  title={Mixed-integer PDE-constrained optimal control of gas networks},
  author={Hahn, Mirko and Leyffer, Sven and Zavala, Victor M},
  journal={Mathematics and Computer Science},
  volume={113},
  year={2017}
}

@article{frangioni2006solving,
  title={Solving nonlinear single-unit commitment problems with ramping constraints},
  author={Frangioni, Antonio and Gentile, Claudio},
  journal={Operations Research},
  volume={54},
  number={4},
  pages={767--775},
  year={2006},
  publisher={INFORMS}
}

@article{biegler1997systematic,
  title={Systematic methods for chemical process design},
  author={Biegler, Lorenz T and Grossmann, Ignacio E and Westerberg, Arthur W},
  year={1997},
  publisher={Prentice Hall, Old Tappan, NJ (United States)}
}

@book{mansini2015linear,
  title={Linear and Mixed Integer Programming for Portfolio Optimization},
  author={Mansini, Renata and Ogryczak, W{\l}odzimierz and Speranza, M. Grazia and EURO: The Association of European Operational Research Societies},
  volume={21},
  year={2015},
  publisher={Springer}
}

@article{sridhar2014models,
  title={Models and solution techniques for production planning problems with increasing byproducts},
  author={Sridhar, Srikrishna and Linderoth, Jeffrey and Luedtke, James},
  journal={Journal of Global Optimization},
  volume={59},
  number={2},
  pages={597--631},
  year={2014},
  publisher={Springer}
}

@book{berthold2015heuristic,
  title={Heuristic algorithms in global MINLP solvers},
  author={Berthold, Timo},
  year={2015},
  publisher={Verlag Dr. Hut}
}

\newpage
\appendix

\section{Details of Integer Correction Layers}
\label{app:corr}

This appendix provides additional details for the integer correction layers introduced in Section~\ref{subsec:int_corr}. We include explicit formulations, gradient derivations, and smoothness analyses for both \textit{Rounding Classification (RC)} and \textit{Learnable Threshold (LT)} variants.

\subsection{Details of Rounding Classification}

\paragraph{Forward Pass.}
The key step of the \textit{Rounding Classification} (RC) approach is performed in line~6 of~\Cref{algo:l2o}. For the integer variables, RC applies a stochastic soft-rounding mechanism to the neural network output ${\bm{h}}_z = \bm{\delta}_{\Theta_2} (\bar{\bm{x}}, {\bm \xi})$, generating a binary decision vector ${\bm{b}} \in \{0,1\}^{n_z}$ that determines whether each fractional component of $\bar{\bm{x}}_z$ is rounded up or down.
To enable gradient-based training, RC employs the Gumbel--Sigmoid. This stochastic relaxation perturbs the logits ${\bm{h}}_z$ with Gumbel noise and passes them through a temperature-controlled Sigmoid, yielding soft probabilities ${\bm{v}} \in (0,1)^{n_z}$. During the forward pass, the final discrete rounding decisions are obtained as
$
{\bm{b}} = \mathbb{I}({\bm{v}} > 0.5), \,
\hat{\bm{x}}_z = \lfloor \bar{\bm{x}}_z \rfloor + {\bm{b}}.
$

\paragraph{Backward Pass.}
Because the binarization operation is non-differentiable, the gradient of ${\bm b}$ w.r.t. the  ${\bm v}$  is approximated with STE.
Hence, the final gradient of ${\bm b}$ w.r.t. the input logit ${\bm h}$ is approximated by the gradient of the Gumbel-Sigmoid given as
$
    \frac{\partial \bm{b}}{\partial \bm{h}} := \frac{\partial \bm{v}}{\partial \bm{h}} = \frac{1}{\tau} \cdot \bm{v} \odot (1 - \bm{v})
$,
where the gradient expression is computed elementwise, with $\odot$ denoting the Hadamard (elementwise) product.
This formulation allows gradients to flow during backpropagation despite the non-differentiable binarization. 

\paragraph{Lipschitz Smoothness $L_{\varphi}$ of the Gradient.}
Now applying the chain rule, the gradient of the binary decision vector $\bm{b}$ with respect to the inputs $\bar{\bm{x}}_z$ of the network $\delta_{\Theta_2}$ is given by
$$
\frac{\partial \bm{b}}{\partial \bar{\bm{x}}_z}
= \frac{\partial \bm{b}}{\partial \bm{h}_z} \cdot \frac{\partial \bm{h}_z}{\partial \bar{\bm{x}}_z}
\approx \frac{\partial \bm{v}}{\partial \bm{h}_z} \cdot \frac{\partial \delta_{\Theta_2}(\bar{\bm{x}}, \bm{\xi})}{\partial \bar{\bm{x}}_z}
= \frac{1}{\tau} \cdot \bm{v} \odot (1 - \bm{v})  \cdot \frac{\partial \delta_{\Theta_2}(\bar{\bm{x}}, \bm{\xi})}{\partial \bar{\bm{x}}_z}.
$$
 Meanwhile, since the floor operation $\lfloor \bar{\bm{x}}_z \rfloor$ is non-differentiable, we again apply the STE by treating it as the identity function during backpropagation:
$ 
\frac{\partial \lfloor \bar{\bm{x}}_z \rfloor}{\partial \bar{\bm{x}}_z} := \bm{I},
$
hence,  contributing a Lipschitz constant of $1$ to $\bar{\bm{x}}_z$.

The Jacobian of the RC correction layer to the neural network input  $\bar{\bm{x}}$ can thus be approximated as:
$$
\nabla_{\bar{\bm{x}}} \varphi_{\Theta_2}(\bar{\bm{x}}) \approx \bm{I} + \frac{1}{\tau} \cdot \operatorname{diag}(\bm{v} \odot (1 - \bm{v})) \cdot \nabla_{\bar{\bm{x}}} \delta_{\Theta_2}(\bar{\bm{x}}, \bm{\xi}).
$$

We now analyze the Lipschitz constant of this gradient map. Define the scalar function. Let us define the scalar function:
$
g(h) = \frac{1}{\tau} \cdot \sigma\left( z \right) \left( 1 - \sigma\left(z \right) \right),
$
where $\sigma(\cdot) $is the sigmoid function, $z = \frac{h + \epsilon_1 - \epsilon_2}{\tau}$, and $\epsilon_1, \epsilon_2 \sim \text{Gumbel}(0, 1) $ are independent samples. 
This function corresponds to the elementwise gradient of the Gumbel-Sigmoid output with respect to its input logit $h$, under the STE approximation we use in the backward pass. It reflects the sensitivity of the soft relaxation $\bm{v}$ to changes in the perturbed logits $\bm{h}$. The shape and boundedness of $g(h)$ directly influence the stability and smoothness of our optimization process.

The maximum absolute value of this derivative over all $z \in \mathbb{R}$ determines the Lipschitz constant. The product $\sigma(z)(1 - \sigma(z))(1 - 2\sigma(z))$ attains its maximum absolute value at $\sigma(z) = \frac{1}{2} \pm \frac{1}{2\sqrt{3}}$, yielding
$
|g'(h)| \leq \frac{1}{6\sqrt{3}\tau^2} \approx \frac{0.0962}{\tau^2}.
$
Hence, the Lipschitz constant of the STE-approximated gradient of the Gumbel-Sigmoid layer is bounded by
$
L_{\text{Gumbel}} \leq \frac{0.0962}{\tau^2}.
$
This implies that as the temperature $\tau $ decreases (to make the sampling sharper), the gradient becomes more sensitive to changes in $h$, which can affect training stability. 

We now estimate the Lipschitz constant of the approximate Jacobian $\nabla_{\bar{\bm{x}}} \varphi_{\Theta_2}(\bar{\bm{x}})$, which is central to the convergence analysis of the integer feasibility projection (see Theorem~\ref{thm:relu_exterior_convergence}). Since both the Gumbel modulation term and the neural network are Lipschitz continuous, the local Lipschitz constant is bounded by
$
L_{\varphi} \leq L_{\text{Gumbel}} \cdot \left\| \nabla_{\bar{\bm{x}}} \delta_{\Theta_2}(\bar{\bm{x}}, \bm{\xi}) \right\|,
$
where $\left\| \cdot \right\|$ denotes the spectral (operator) norm, i.e., the largest singular value of the Jacobian.
For a global Lipschitz estimate, we have
$
L_{\varphi}^{\text{global}} \leq \frac{0.0962}{\tau^2} \cdot \sup_{\bar{\bm{x}}} \left\| \nabla_{\bar{\bm{x}}} \delta_{\Theta_2}(\bar{\bm{x}}, \bm{\xi}) \right\|.
$
This bound highlights how the temperature parameter $\tau$ and the smoothness of the logit network jointly affect the stability of the correction layer. In our setup with $\tau = 1$, this yields a concrete local bound $L_{\varphi} \leq 0.0962 \cdot \left\| \nabla_{\bar{\bm{x}}} \delta_{\Theta_2}(\bar{\bm{x}}, \bm{\xi}) \right\|$.

\subsection{Details of Learnable Threshold}

\paragraph{Forward Pass.}
The \textit{Learnable Threshold} (LT) approach, detailed in of~\Cref{algo:l2o}, provides an alternative correction strategy. Instead of relying on probability as in RC, LT learns to predict a threshold vector ${\bm h}^i \in [0, 1]^{n_z}$ by applying a Sigmoid activation, which guides rounding decisions for each integer variable. These thresholds ${\bm h}$ are then compared against the fractional part of the relaxed integer variables. Specifically, a variable is rounded up if its fractional part $\bar{{\bm x}}_z - \hat{{\bm x}}_z$ exceeds the threshold ${\bm h}$, and rounded down otherwise. Thus, the binary decision in the forward pass is computed as
$
{\bm b} = \mathbb{I} (\bar{{\bm x}}_z - \hat{{\bm x}}_z - {\bm h} > 0).
$

\paragraph{Backward Pass.}
Although the forward pass applies a hard threshold, the backward pass approximates the gradient from the following smoothed Sigmoid surrogate
$
{\bm v} = \frac{1}{1 + \exp\left( - \beta \cdot (\bar{{\bm x}}_z - \hat{{\bm x}}_z - {\bm h}) \right)},
$
where $\beta > 0$ controls the steepness of the approximation. A higher $\beta$ yields sharper transitions. We use $\beta = 10$ in our experiments.
Thus, the approximated partial derivatives of ${\bm b}$ w.r.t. the threshold ${\bm h}$ are
$
    \frac{\partial \bm{b}}{\partial \bm{h}} := \frac{\partial \bm{v}}{\partial \bm{h}} = -\beta \cdot {\bm v} \odot (1 - {\bm v}).
$

\paragraph{Lipschitz Smoothness $L_{\varphi}$ of the Gradient.}
Since the gradient of the LT correction layers is approximated with a scaled sigmoid, let's analyze its maximum slope. 
The product $\sigma(z)(1 - \sigma(z))$ is maximized at $z = 0$, where
$ 
\sigma(0) = 0.5 \, \Rightarrow \, \sigma(0)(1 - \sigma(0)) = 0.25.
$
So the maximum value of the derivative of the scaled sigmoid function is
\[
\max_x \left| \frac{d}{dx} \sigma(\beta x) \right| = \beta \cdot \max_z \sigma(z)(1 - \sigma(z)) = \beta \cdot 0.25 = \frac{\beta}{4}.
\]

Hence, similar to the RC method, the local Lipschitz constant of the approximate Jacobian $\nabla_{\bar{\bm{x}}} \varphi_{\Theta_2}$ for the LT correction layer is bounded by:
$ 
L_{\varphi} \leq \frac{\beta}{4} \cdot  \left\| \nabla_{\bar{\bm{x}}} \delta_{\Theta_2}(\bar{\bm{x}}, \bm{\xi}) \right\|,
$
while the global Lipschitz constant is bounded by
$
L_{\varphi}^{\text{global}} \leq \frac{\beta}{4} \cdot \sup_{\bar{\bm{x}}} \left\| \nabla_{\bar{\bm{x}}} \delta_{\Theta_2}(\bar{\bm{x}}, \bm{\xi}) \right\|.
$
In our setup with $\beta = 10$, this yields the concrete local bound 
$ L_{\varphi} \leq 2.5 \cdot \left\| \nabla_{\bar{\bm{x}}} \delta_{\Theta_2}(\bar{\bm{x}}, \bm{\xi}) \right\|$.

\section{Additional Convergence Analysis}
\label{app:theory}

\subsection{Asymptotic Convergence Based on the Łojasiewicz Inequality}
\label{subapp:proof_Łoj}

Theorem~\ref{thm:relu_exterior_convergence} relies on mild regularity assumptions, as noted in Remark~\ref{rmk:regularity_lipschitz}. Alternatively, convergence can be shown using the Łojasiewicz inequality by leveraging the analytic or subanalytic structure of ${\bm g}$ and $\varphi$, without requiring explicit smoothness or curvature bounds.

\begin{theorem}[Feasibility Convergence of Integer Feasibility Projection via Łojasiewicz Inequality]
\label{thm:loja_convergence}
Let $\mathcal{V}({\bm x}) = \sum_{j=1}^{n_c} \max(0, g_j(\varphi({\bm x})))$, where ${\bm g} \in C^\omega$ (real analytic) and $\varphi \in C^\infty$ are composed with piecewise-linear ReLU, and suppose $\mathcal{V}({\bm x}) > 0$ on the region $\mathcal{D} \subset \mathbb{R}^n$.
Then $\mathcal{V}$ is subanalytic and satisfies the Łojasiewicz gradient inequality at every critical point ${\bm x}^* \in \mathcal{D}$. As a result, for any initialization ${\bm x}^{(0)} \in \mathcal{D}$, the gradient descent iterates
$
{\bm x}^{(k+1)} = {\bm x}^{(k)} - \eta \nabla \mathcal{V}({\bm x}^{(k)}), \quad \eta \in (0, 1/L],
$
converge to a critical point ${\bm x}^\infty \in \mathcal{D}$, and $\mathcal{V}({\bm x}^{(k)}) \to \mathcal{V}({\bm x}^\infty)$.
If further, all critical points ${\bm x}^* \in \mathcal{D}$ with $\mathcal{V}({\bm x}^*) > 0$ are not local minima of $\mathcal{V}$, then
$
\lim_{k \to \infty} \mathcal{V}({\bm x}^{(k)}) = 0.
$
\end{theorem}

\begin{proof}{}
 We divide the proof of Theorem~\ref{thm:loja_convergence} into two parts: i) convergence of gradient descent to a critical point, and  ii) convergence to feasibility under a non-minimality assumption.
\paragraph{Subanalyticity and Łojasiewicz inequality.} 
Since ${\bm g}: \mathbb{R}^n \to \mathbb{R}^{n_c}$ is real analytic and $\varphi: \mathbb{R}^n \to \mathbb{R}^n$ is $C^\infty$, their composition ${\bm g} \circ \varphi$ is real analytic as well. The ReLU function $\max(0, z_j)$ is piecewise analytic and semialgebraic.
Therefore, the function $\mathcal{V}({\bm x}) = \sum_{j=1}^{n_c} \max(0, g_j(\varphi({\bm x})))$ is a finite sum of semialgebraic (hence subanalytic) functions composed with analytic mappings, and is itself a \emph{subanalytic} and $C^1$ function on the open region
$
\mathcal{D} := \{ {\bm x} \in \mathbb{R}^n : \mathcal{V}({\bm x}) > 0 \}.
$
By a standard result in nonsmooth analysis (see~\cite{bolte2007_clarke}), every $C^1$ subanalytic function satisfies the \emph{Łojasiewicz gradient inequality} at all its critical points. That is, for each critical point ${\bm x}^* \in \mathcal{D}$, there exist constants $C > 0$, $\theta \in [0,1)$, and a neighborhood $\mathcal{U} \subset \mathcal{D}$ of ${\bm x}^*$ such that
$
\| \nabla \mathcal{V}({\bm x}) \| \ge C (\mathcal{V}({\bm x}) - \mathcal{V}({\bm x}^*))^\theta, \, \forall x \in \mathcal{U}.
$

\paragraph{Convergence of gradient descent to a critical point.}
Let ${\bm x}^{(k+1)} = {\bm x}^{(k)} - \eta \nabla \mathcal{V}({\bm x}^{(k)})$ be the gradient descent update with constant step size $\eta \in (0, 1/L]$, where $L$ is a Lipschitz constant for $\nabla \mathcal{V}$ on compact subsets of $\mathcal{D}$ (as shown in Theorem~\ref{thm:relu_exterior_convergence}).
The descent lemma implies
$
\mathcal{V}({\bm x}^{(k+1)}) \le \mathcal{V}({\bm x}^{(k)}) - \frac{\eta}{2} \| \nabla \mathcal{V}({\bm x}^{(k)}) \|^2.
$
Since $\mathcal{V}({\bm x}^{(k)}) \ge 0$ and is non-increasing, it converges to a finite limit $\mathcal{V}^\star \ge 0$. Furthermore, summing the descent inequality gives
$
\sum_{k=0}^\infty \| \nabla \mathcal{V}({\bm x}^{(k)}) \|^2 < \infty, \quad \Rightarrow \quad \lim_{k \to \infty} \| \nabla \mathcal{V}({\bm x}^{(k)}) \| = 0.
$
Hence, the sequence $\{{\bm x}^{(k)}\}$ has at least one accumulation point ${\bm x}^\infty \in \mathcal{D}$ with $\nabla \mathcal{V}({\bm x}^\infty) = 0$, i.e., a critical point. Since $\mathcal{V}$ is subanalytic and satisfies the Łojasiewicz inequality at all critical points, the full sequence ${\bm x}^{(k)} \to {\bm x}^\infty$, as shown by the convergence theorem for gradient descent on Łojasiewicz functions (see~\cite{attouch2013_convergence}).

\paragraph{Convergence to feasibility under non-minimality.}
Suppose now that ${\bm x}^\infty \in \mathcal{D}$ is a critical point with $\mathcal{V}({\bm x}^\infty) > 0$. Then, by the Łojasiewicz inequality
$
\| \nabla \mathcal{V}({\bm x}^{(k)}) \| \ge C (\mathcal{V}({\bm x}^{(k)}) - \mathcal{V}({\bm x}^\infty))^\theta.
$
But this contradicts the fact that $\| \nabla \mathcal{V}({\bm x}^{(k)}) \| \to 0$ unless $\mathcal{V}({\bm x}^{(k)}) \to \mathcal{V}({\bm x}^\infty)$. Thus $\mathcal{V}({\bm x}^{(k)}) \to \mathcal{V}({\bm x}^\infty) > 0$.
To rule this out, we assume that no such point ${\bm x}^\infty \in \mathcal{D}$ with $\mathcal{V}({\bm x}^\infty) > 0$ is a \emph{local minimum} of $\mathcal{V}$. That is, every critical point with $\mathcal{V}({\bm x}^*) > 0$ is a saddle or otherwise unstable.
This ensures that the gradient descent trajectory cannot converge to any such point ${\bm x}^\infty$. Hence, it must converge to a point where $\mathcal{V}({\bm x}^\infty) = 0$, i.e., feasibility is achieved at the limit
$
\lim_{k \to \infty} \mathcal{V}({\bm x}^{(k)}) = 0.
$
\end{proof}

\begin{remark}[Relationship Between Theorems~\ref{thm:relu_exterior_convergence} and~\ref{thm:loja_convergence}]
Theorem~\ref{thm:relu_exterior_convergence} shows that under a mild structural condition, gradient descent converges to feasibility with vanishing gradient norm.
Theorem~\ref{thm:loja_convergence} provides a complementary convergence result using the Łojaśiewicz gradient inequality, which is automatically satisfied due to the analytic structure of the composite function $\mathcal{V}$ representing our integer feasibility projection. This yields convergence to feasibility without requiring explicit curvature conditions. Hence, generalizing the convergence guarantees to non-smooth but subanalytic neural surrogates of the rounding operation $\varphi$, such as those with  ReLU activation functions.
However, this comes at the expense of losing explicit rates.
\end{remark}

\subsection{Robust Feasibility Convergence under Local Minima}
\label{subapp:proof_local}

While Theorem~\ref{thm:loja_convergence} guarantees convergence to feasibility under a general non-minimality condition, we now show that this condition is satisfied for a wide class of problematic critical points, namely, strict local minima of the integer correction layer $\varphi$.

\begin{theorem}[Strict Local Minima of $\varphi$ Do Not Trap Gradient Descent with ReLU--L1 Penalty]
\label{thm:local_minima_h}
Let $\varphi: \mathbb{R}^n \to \mathbb{R}^n$ be twice continuously differentiable, and let $\bm{g}: \mathbb{R}^n \to \mathbb{R}^{n_c}$ be continuously differentiable. Define the ReLU-based penalty function
$
\mathcal{V}(\bm{x}) := \sum_{j=1}^{n_c} \max(0, g_j(\varphi(\bm{x}))).
$
Suppose $\bm{x}^* \in \mathbb{R}^n$ is a strict local minimizer of $\varphi$, i.e.,
$
\nabla \varphi(\bm{x}^*) = 0, 
\, 
\nabla^2 \varphi(\bm{x}^*) \succ 0,
$
and that at least one constraint is violated at $\bm{x}^*$, i.e., $g_j(\varphi(\bm{x}^*)) > 0$ for some $j$. Then:
\begin{enumerate}
    \item $\nabla \mathcal{V}(\bm{x}^*) = 0$, so $\bm{x}^*$ is a stationary point of $\mathcal{V}$;
    \item $\bm{x}^*$ is not a local minimum of $\mathcal{V}$;
    \item For small perturbations $\bm{\delta}$, $\nabla \mathcal{V}(\bm{x}^* + \bm{\delta}) \ne 0$ generically;
    \item Therefore, gradient descent initialized near $\bm{x}^*$ will escape from it.
\end{enumerate}
\end{theorem}

\begin{proof}
{}Now we prove Theorem~\ref{thm:local_minima_h} in the following steps.

\paragraph{Stationarity.}
Define the active index set at $\bm{x}^*$ as
$
I_{\bm{x}} := \{ j : g_j(\varphi(\bm{x}^*)) > 0 \}.
$
Because $g_j(\varphi(\bm{x}^*)) > 0$ for $j \in I_{\bm{x}}$, the ReLU terms are smooth in a neighborhood of $\bm{x}^*$, and thus locally
$
\mathcal{V}(\bm{x}) = \sum_{j \in I_{\bm{x}}} g_j(\varphi(\bm{x})).
$
By the chain rule,
$
\nabla \mathcal{V}(\bm{x}) = \sum_{j \in I_{\bm{x}}} \nabla \varphi(\bm{x})^\top \nabla g_j(\varphi(\bm{x})).
$
Since $\nabla \varphi(\bm{x}^*) = 0$, we have $\nabla \mathcal{V}(\bm{x}^*) = 0$. Hence $\bm{x}^*$ is a stationary point of $\mathcal{V}$.

\paragraph{Not a Local Minimum of $\mathcal{V}$.}
We now examine the curvature of $\mathcal{V}$ near $\bm{x}^*$.  
The Hessian expands as
$$
\nabla^2 \mathcal{V}(\bm{x}^*) 
= \sum_{j \in I_{\bm{x}}} \nabla^2 (g_j \circ \varphi)(\bm{x}^*)
= \sum_{j \in I_{\bm{x}}} 
\nabla^2 \varphi(\bm{x}^*)^\top \nabla g_j(\varphi(\bm{x}^*))
+ \nabla \varphi(\bm{x}^*)^\top \nabla^2 g_j(\varphi(\bm{x}^*)) \nabla \varphi(\bm{x}^*).
$$
Since $\nabla \varphi(\bm{x}^*) = 0$, the second term vanishes, yielding
$
\nabla^2 \mathcal{V}(\bm{x}^*) 
= \sum_{j \in I_{\bm{x}}} \nabla^2 \varphi(\bm{x}^*)^\top \nabla g_j(\varphi(\bm{x}^*)).
$
Each term scales the positive-definite curvature $\nabla^2 \varphi(\bm{x}^*)$ by the direction of $\nabla g_j(\varphi(\bm{x}^*))$. Unless all $\nabla g_j$ are perfectly aligned with this curvature, some directions will introduce negative or zero curvature, rendering $\nabla^2 \mathcal{V}(\bm{x}^*)$ indefinite. Hence $\bm{x}^*$ cannot be a local minimum of $\mathcal{V}$.

\begin{remark}[Directional Sufficient Condition for Non-Minimality]
\label{rmk:directional_condition}
A sufficient condition for ${\bm x}^*$ to \emph{not} be a local minimum of $\mathcal{V}$ is the existence of a direction 
${\bm v} \in \mathbb{R}^n$ and an active constraint $j \in I_{\bm x}$ such that
$
{\bm v}^\top \nabla^2 \varphi({\bm x}^*) {\bm v} > 0$, and 
${\bm v}^\top \nabla g_j(\varphi({\bm x}^*)) < 0.
$
The first condition implies that $\varphi$ is locally convex along ${\bm v}$, 
while the second indicates that the violated constraint decreases in that same direction.
Together, they ensure negative curvature in the composite penalty $\mathcal{V}$,
i.e., ${\bm v}^\top \nabla^2 \mathcal{V}({\bm x}^*) {\bm v} < 0$.
Hence, even if $\varphi$ has a strict local minimum at ${\bm x}^*$, 
the ReLU--L1 penalty $\mathcal{V}$ remains non-minimizing and allows gradient descent to escape.
\end{remark}

\paragraph{Perturbation Analysis.}
For a small perturbation $\bm{\delta}$, we have
$
\nabla \varphi(\bm{x}^* + \bm{\delta})
= \nabla^2 \varphi(\bm{x}^*) \bm{\delta} + o(\|\bm{\delta}\|),
$
which is nonzero since $\nabla^2 \varphi(\bm{x}^*) \succ 0$.  
By continuity of $\nabla g_j$, we obtain
$
\nabla \mathcal{V}(\bm{x}^* + \bm{\delta})
= \sum_{j \in I_{\bm{x}}}
\nabla \varphi(\bm{x}^* + \bm{\delta})^\top
\nabla g_j(\varphi(\bm{x}^* + \bm{\delta}))
\neq 0
$
for almost all small $\bm{\delta}$.

\paragraph{Escape from ${\bm x}^*$.}  
Since $\nabla \mathcal{V}({\bm x}) \ne 0$ in a neighborhood around ${\bm x}^*$, gradient descent will not be trapped at ${\bm x}^*$. Any initialization near ${\bm x}^*$ will result in descent away from the point.
\end{proof}

\begin{remark}[Generic Nondegeneracy of Constraint Gradients and Curvature Alignment]
\label{rmk:generic_nonvanishing}
The following structural assumptions are generic and sufficient for ensuring convergence to feasibility:

\begin{enumerate}
    \item \textbf{Nonvanishing gradients.}  
    For real analytic or $C^1$ non-constant $g_j$, the set  
    $Z_j := \{ {\bm x} : g_j({\bm x}) > 0, \nabla g_j({\bm x}) = 0 \}$  
    has measure zero, implying that $\nabla g_j(\varphi({\bm x}^*)) \ne 0$ almost surely for any violated constraint.

    \item \textbf{Gradient–curvature misalignment.}  
    Since $\nabla g_j$ vary independently, they are rarely aligned with the curvature of $\varphi$.  
    Consequently, the composite Hessian  
    $\nabla^2 \mathcal{V}({\bm x}^*) 
     = \sum_{j \in I_{{\bm x}^*}} \nabla^2 (g_j \circ \varphi)({\bm x}^*)$
    is generically indefinite whenever a direction ${\bm v}$ satisfies  
    ${\bm v}^\top \nabla^2 \varphi({\bm x}^*) {\bm v} > 0$ and  
    ${\bm v}^\top \nabla g_j(\varphi({\bm x}^*)) < 0$.
\end{enumerate}

Hence, under mild and generic assumptions, a strict local minimum of $\varphi$ with $\mathcal{V}({\bm x}^*) > 0$ cannot be a local minimum of $\mathcal{V}$, and gradient descent will escape.
\end{remark}

\subsection{Exclusion of Degenerate Convergence Scenarios}
\label{subapp:degen}

While the preceding results establish convergence to feasibility under mild regularity assumptions, it is important to examine whether pathological critical points or plateau-like regions could still prevent convergence. This section characterizes and rules out such degenerate scenarios.

\begin{remark}[Plateau Behavior and Subanalyticity]
\label{rmk:plateaus_subanalyticity}
One may wonder whether the penalty function
$ \mathcal{V}({\bm x}) = \sum_{j=1}^{n_c} \max(0, g_j(\varphi({\bm x})))$
can exhibit plateau-like behavior, i.e., regions where $\mathcal{V}$ remains constant and positive, within the infeasible region $\mathcal{D} := \{ {\bm x} : \mathcal{V}({\bm x}) > 0 \}$. We address two such scenarios below.

\textbf{Flat critical manifolds.} A flat manifold is a set $M \subset \mathcal{D}$ where $\nabla \mathcal{V}({\bm x}) = 0$ and $\mathcal{V}({\bm x}) = c > 0$ for all $x \in M$. These sets could trap gradient descent if they existed with positive measure. However, because $\mathcal{V}$ is subanalytic and $C^1$ on $\mathcal{D}$, it satisfies the Kurdyka--Łojasiewicz (KL) inequality near all critical points~\citep{bolte2007_clarke}. This rules out the existence of non-isolated flat critical sets unless $\mathcal{V}$ is locally constant, which we now argue is also structurally implausible.
Moreover, known convergence results for KL functions~\citep{attouch2013_convergence} imply that gradient descent cannot asymptotically converge to a non-isolated flat critical manifold unless it is initialized there. In typical smooth machine learning problems, such events occur with probability zero under random initialization.
Therefore, the subanalyticity of $\mathcal{V}$ implies that flat critical manifolds are unstable under gradient descent. 

\textbf{Constant regions.} Suppose, for contradiction, that $\mathcal{V}({\bm x}) = c > 0$ on an open subset $U \subset \mathcal{D}$. Then each active term $j \in I_{\bm x} := \{ j : g_j(\varphi({\bm x})) > 0 \}$ must be constant over $U$, implying that the compositions $g_j \circ \varphi$ are locally constant. This in turn forces their gradients to vanish: $\nabla (g_j \circ \varphi)({\bm x}) = 0$ for all $x \in U$. Unless $g_j \circ \varphi$ is identically constant—a non-generic scenario—this condition fails on open sets.

\textbf{Implication.} Together, subanalytic regularity and mild structural assumptions that $g_j \circ \varphi$ are not constant functions ensure that $\mathcal{V}$ cannot be locally constant on any open subset of $\mathcal{D}$. Therefore, genuine plateaus or flat manifolds that could trap gradient descent do not arise in typical settings.
\end{remark}

\begin{remark}[Critical Points and Optimization Challenges]
\label{rmk:critical_point_challenges}
The penalty function $\mathcal{V}({\bm x})$ is piecewise smooth and subanalytic on the infeasible region $\mathcal{D} := \{{\bm x} : \mathcal{V}({\bm x}) > 0\}$. A natural question is whether gradient descent could become trapped at infeasible critical points.
 Two classes of critical points could, in principle, obstruct convergence.

\textbf{Non-degenerate local minima.} Points ${\bm x}^* \in \mathcal{D}$ where $\nabla \mathcal{V}({\bm x}^*) = 0$, $\nabla^2 \mathcal{V}({\bm x}^*) \succ 0$, and $\mathcal{V}({\bm x}) > \mathcal{V}({\bm x}^*)$ locally. These may arise if the active constraint set $I_{{\bm x}^*} := \{ j : g_j(\varphi({\bm x}^*)) > 0 \}$ is fixed and the composite Hessian
    $
    \nabla^2 \mathcal{V}({\bm x}^*) = \sum_{j \in I_{{\bm x}^*}} \nabla^2 (g_j \circ \varphi)({\bm x}^*)
    $
    is positive definite. However, such configurations require fine alignment between $\nabla g_j$ and the curvature of $\varphi$, which is highly non-generic.

\textbf{Flat saddles.} Isolated critical points where $\nabla \mathcal{V}({\bm x}^*) = 0$ and the Hessian is degenerate (e.g., zero eigenvalues). These may occur under degeneracy or saturation in $\varphi$, or when multiple $\nabla g_j(\varphi({\bm x}))$ vanish. While subanalyticity rules out flat critical \emph{manifolds}, it does not preclude such isolated saddles.

\subsubsection{Taxonomy of Feasibility Convergence}
Feasibility convergence is ensured by one of two assumptions:
\begin{itemize}
    \item \emph{Structural:} If every ${\bm x}^* \in \mathcal{D}$ with $\mathcal{V}({\bm x}^*) > 0$ has some active constraint $g_j(\varphi({\bm x}^*)) > 0$ with $\nabla g_j(\varphi({\bm x}^*)) \ne 0$, then $\nabla \mathcal{V}({\bm x}^*) \ne 0$, so infeasible stationary points are excluded. This condition appears in Theorem~\ref{thm:relu_exterior_convergence} and rules out non-degenerate local minima.

    \item \emph{Dynamical:} Alternatively, assume that no infeasible stationary point ${\bm x}^* \in \mathcal{D}$ attracts nearby trajectories under gradient descent
    $
    {\bm x}^{(k+1)} = {\bm x}^{(k)} - \eta \nabla \mathcal{V}({\bm x}^{(k)}).
    $
    That is, for every neighborhood $B_\rho({\bm x}^*) := \{ {\bm x} : \|{\bm x} - {\bm x}^*\| < \rho \}$, some ${\bm x} \in B_\rho({\bm x}^*) \cap \mathcal{D}$ generates a trajectory that does not converge to ${\bm x}^*$. This allows such critical points to exist but ensures they do not trap iterates, especially relevant for flat saddles or degenerate cases not excluded structurally.
\end{itemize}

\textbf{Implication.} 
Whether through gradient non-vanishing or non-attraction, infeasible critical points are generically avoided. Combined with the subanalytic geometry of $\mathcal{V}$, these conditions help explain why gradient descent almost always escapes infeasible regions in practice. Moreover, stochastic methods and perturbation-based algorithms~\citep{jin2017escape} have been shown to escape strict and flat saddles in polynomial time under mild conditions.

\end{remark}

\begin{proposition}[Characterization of Critical Point Behavior]
\label{prop:critical_point_taxonomy}
Let ${\bm x}^* \in \mathcal{D} := \{ {\bm x} \in \mathbb{R}^n : \mathcal{V}({\bm x}) > 0 \}$ be a stationary point of the ReLU-penalized objective 
$
\mathcal{V}({\bm x}) = \sum_{j=1}^{n_c} \max(0, g_j(\varphi({\bm x}))).
$
Then, the behavior of gradient descent near ${\bm x}^*$ depends on the type of critical point as given in Table~\ref{tab:critical_points}.
\end{proposition}

\begin{table}[htb!]
\vspace{-6pt}
\centering
\scriptsize
\setlength{\tabcolsep}{6pt}
\renewcommand{\arraystretch}{0.85}
\caption{Gradient descent behavior at different types of critical points of $\mathcal{V}$ or $\varphi$.}
\label{tab:critical_points}

\resizebox{0.85\linewidth}{!}{
\begin{tabular}{lccc}
\toprule
\textbf{Critical Point Type} 
& \textbf{Addressed in} 
& \textbf{GD Converges?} 
& \textbf{Feasible?} \\
\midrule
Strict local minimum of $\varphi$             
& Thm~\ref{thm:local_minima_h}               
& No                                     
& \textcolor{green}{Yes} \\

Strict saddle of $\mathcal{V}$               
& Thm~\ref{thm:loja_convergence}            
& No (a.s.)                       
& \textcolor{green}{Yes} \\

Non-isolated saddle of $\mathcal{V}$         
& Thm~\ref{thm:loja_convergence}            
& No (generic)                         
& \textcolor{green}{Yes} \\

Flat critical manifold ($\mathcal{V} > 0$) 
& Rem.~\ref{rmk:plateaus_subanalyticity}    
& No (subanalytic)                         
& \textcolor{green}{Yes} \\

Locally constant region of $\mathcal{V}$     
& Rem.~\ref{rmk:plateaus_subanalyticity}    
& No (subanalytic)                         
& \textcolor{green}{Yes} \\

Non-degenerate saddle of $\mathcal{V}$       
& Thm~\ref{thm:loja_convergence}            
& No (stable)           
& \textcolor{green}{Yes} \\

Non-degenerate local minimum of $\mathcal{V}$ 
& Rem.~\ref{rmk:critical_point_challenges}  
& Yes (rare)                  
& \textcolor{red}{No} \\

Flat saddle of $\mathcal{V}$                 
& Rem.~\ref{rmk:critical_point_challenges}  
& Unclear                                  
& Open \\

Degenerate saddle of $\mathcal{V}$           
& Rem.~\ref{rmk:critical_point_challenges}  
& Unclear                                  
& Open \\

\bottomrule
\end{tabular}
}
\vspace{-6pt}
\end{table}

In summary:
\begin{itemize}
    \item \textbf{Feasibility convergence justification.} Infeasible critical points are either excluded structurally (via gradient non-vanishing; see Theorem~\ref{thm:relu_exterior_convergence}) or are assumed to be non-attracting (see Remark~\ref{rmk:critical_point_challenges}). These complementary perspectives explain why convergence to feasibility occurs in practice.

    \item \textbf{Strict local minima of $\varphi$} are ruled out as attractors by Theorem~\ref{thm:local_minima_h}, since they are not minima of $\mathcal{V}$.
    
    \item \textbf{Strict saddles}, \textbf{non-isolated saddles}, \textbf{flat manifolds}, and \textbf{locally constant regions} are generically avoided due to the subanalytic structure of $\mathcal{V}$, which guarantees the Łojasiewicz (KL) property and precludes convergence to non-isolated critical sets, see Theorem~\ref{thm:loja_convergence},  and Remark~\ref{rmk:plateaus_subanalyticity}.

    \item \textbf{Non-degenerate saddles} are generically escaped by gradient descent due to instability in directions of negative curvature.

    \item \textbf{Non-degenerate local minima of $\mathcal{V}$} with $\mathcal{V}({\bm x}^*) > 0$ may exist but are structurally rare and require unlikely gradient-curvature alignment, see Remark~\ref{rmk:critical_point_challenges}.

    \item \textbf{Degenerate or flat saddles} are not ruled out by subanalyticity alone. While rare in practice, they remain an open challenge. Their avoidance may require additional randomness or second-order mechanisms, see Remark~\ref{rmk:critical_point_challenges}.

\end{itemize}

\section{Experimental Details}
\label{app:exp_details}

\subsection{Training Hyperparameters}
All learning-based models are trained with the AdamW optimizer (learning rate $10^{-3}$, weight decay $10^{-4}$), batch size $64$, and a maximum of $200$ epochs, with early stopping based on validation loss. The constraint penalty weight is fixed at $\lambda=100$. The feasibility projection uses a step size $\eta=0.01$ and at most $1000$ iterations.

\subsection{Neural Network Architectures}
The solution-mapping network $\pi_{\Theta_1}$ uses five fully connected layers with ReLU activation.  The correction network $\varphi_{\Theta_2}$ consists of four fully connected layers with ReLU activations, along with Batch Normalization and dropout (rate $0.2$).Hidden widths scale with the problem size:
\begin{itemize}
    \item IQP/INP: widths $\{64, 128, 256, 512, 1024, 2048\}$ depending on the input dimension.
    \item MIRB: widths $\{4, 16, 128, 256, 1024\}$ for problem sizes from $2$ to $10{,}000$ variables.
    \item BLP: input dimension $120$, hidden layer $256$, output dimension $360$.
\end{itemize}

\subsection{Software and Hardware}
Experiments were run on a workstation equipped with two Intel Silver 4216 CPUs (2.1\,GHz), 64\,GB RAM, and an NVIDIA V100 GPU. Neural models were implemented in PyTorch~2.5.0+cu122 and NeuroMANCER~1.5.2. Convex problems were solved with \gurobi{}~11.0.1, and nonconvex problems with \scip{}~9.0.0 + \ipopt{}~3.14.14.

\section{Benchmark Generation Details}
\label{app:bench_details}

\subsection{Integer Quadratic Problems}
\label{subapp:qp_gen}

The integer quadratic problems (IQPs) follow the data-generation protocol of \citet{donti2021dc3}. The matrix $\bm{Q}$ is diagonal with entries drawn from the uniform distribution $\mathcal{U}[0,0.01]$, the linear term is sampled as $\bm{p}\!\sim\!\mathcal{U}[0,0.1]^n$, and the constraint matrix is generated from a Gaussian distribution $\bm{A}\!\sim\!\mathcal{N}(0,\,0.1)^{m\times n}$. The right-hand side $\bm{b}$, which serves as the parametric input, is sampled from $\mathcal{U}[-1,1]^m$. To adapt the problem to the mixed-integer setting, all variables are restricted 
to be integer and equality constraints are removed to maintain feasibility.

\subsection{Integer Nonconvex Problems}
\label{subapp:nc_gen}

The INPs extend IQPs by replacing the linear objective term with a trigonometric component. The parameters $\bm{Q}$, $\bm{p}$, $\bm{A}$, and $\bm{b}$ follow the same distributions as in IQPs.. To introduce additional parametric variability in the feasible region, we draw a vector $\bm{d}$ from the uniform distribution $\mathcal{U}[-0.5,0.5]^m$ and modify the constraint matrix via  $\bm{A}\leftarrow \bm{A} + [\bm{d}, -\bm{d}, 0,\dots,0].$ This small perturbation makes the constraints instance-dependent.

\subsection{Mixed-integer Rosenbrock Problems}
\label{subapp:rb_gen}

In MIRBs, the vectors $\bm{p}, \bm{Q} \in \mathbb{R}^n$ are sampled once from a standard normal distribution and fixed across all instances. The parameters $\bm{a}$ and $b$ vary across instances and serve as the inputs to the learning model: each entry of $\bm{a}$ is drawn from the uniform distribution $\mathcal{U}[0.5,4.5]$, while $b$ is sampled from $\mathcal{U}[1,8]$.

The parameters $(\bm{a},b)$ jointly control the shape of the nonlinear Rosenbrock valley and the geometry of the feasible region, yielding a rich parametric MINLP family that scales from low dimensions to problems with tens of thousands of variables.

\section{Binary Linear Programs}
\label{app:blps}

we additionally evaluate our methods on BLPs from the \emph{Obj Series 1} of the MIP Workshop 2023 Computational Competition \citep{MIPcc23}. This benchmark consists of 50 MILP instances sharing an identical constraint matrix but differing in 120 of the 360 objective coefficients. We reserve all 50 original instances as a held-out test set. To obtain a sufficiently large training distribution with the same structural properties, we generate synthetic instances by uniform sampling new objective vectors $\bm{c}$ within the coefficient range observed in the benchmark, while keeping the constraint matrix fixed.

\Cref{tab:ILP_Stat} summarizes the performance of various optimization methods on the MILP benchmark: Both learning-based methods (RC and LT) demonstrate the ability to generate high-quality feasible solutions efficiently, with RC even surpassing the heuristic-based method N1 in terms of objective value. However, N1 is the fastest method overall, showcasing the robustness and efficiency of the heuristic in the MILP solver. EX achieved the best objective values but required significantly more computation time. Notably, the training time of our method is approximately 120 seconds, making it well-suited for applications with repeated problem-solving.

\begin{table}[htb!]
\centering
\caption{Comparison of Optimization Methods on the MILP. Each method is evaluated on 50 test instances. We report the mean and median objective values (``Obj Mean’’ and ``Obj Median’’), the fraction of feasible solutions (``Feasible’’),  and the average inference or solving time per instance (``Time (Sec)’'). Since the MILP is a minimization problem, smaller objective values are better.
}
\label{tab:ILP_Stat}

\resizebox{0.65\linewidth}{!}{%
\begin{tabular}{l|rrrr}
\toprule[1pt]\midrule[0.3pt]
\textbf{Method} & \textbf{Obj Mean} & \textbf{Obj Median} & \textbf{Feasible} & \textbf{Time (Sec)} \\
\midrule
RC & $9745.90$  & $9763.00$  & $100\%$ & $0.04$ \\
LT & $14149.00$ & $14149.00$ & $100\%$ & $0.04$ \\
EX & $8756.80$  & $8747.00$  & $100\%$ & $28.91$ \\
N1 & $11901.10$ & $11933.00$ & $100\%$ & $0.01$ \\
\midrule[0.3pt]\bottomrule[1pt]
\end{tabular}
} 

\end{table}

\section{Additional Visualizations of Constraint Violations}
\label{app:viol_figures}

This section provides additional visualizations of constraint violations for the RC and LT baselines, complementing the quantitative results reported in the main text. For IQPs and INPs, violations are sparse and of small magnitude, whereas for MIRBs, violations is more pronounced and widespread.

\begin{figure*}[ht]
\vspace{-10pt}
\centering

\begin{minipage}{0.42\textwidth}
    \centering
    \includegraphics[width=\linewidth]{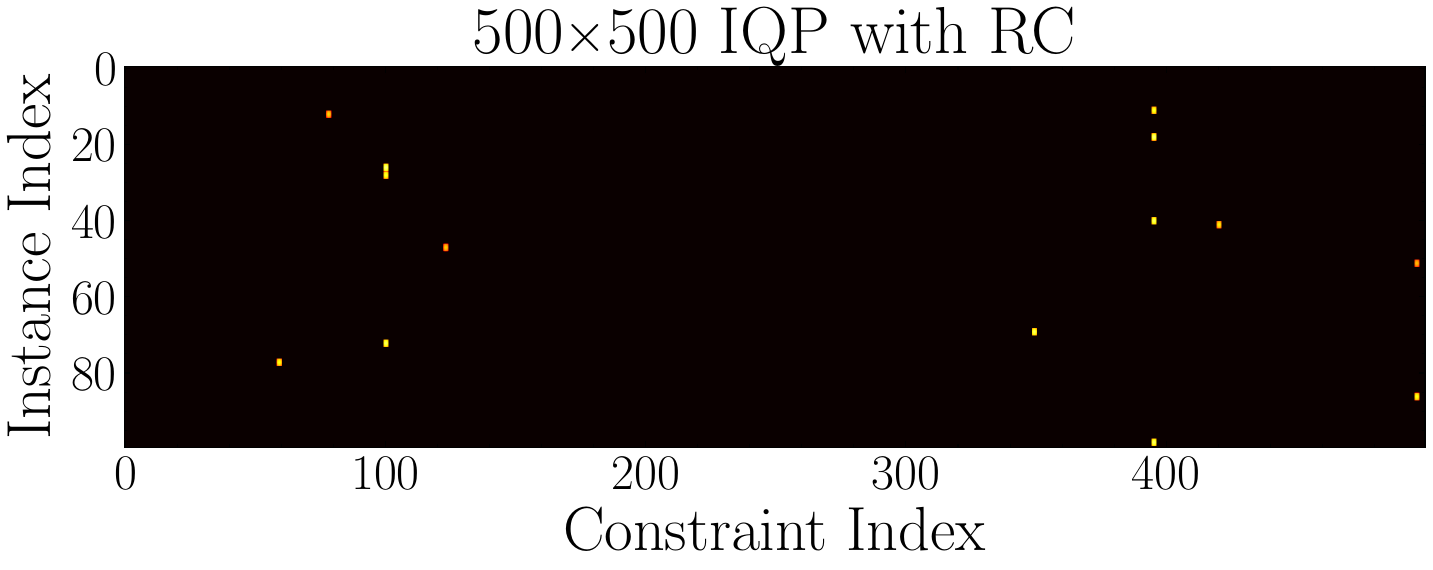}
\end{minipage}
\begin{minipage}{0.49\textwidth}
    \centering
    \includegraphics[width=\linewidth]{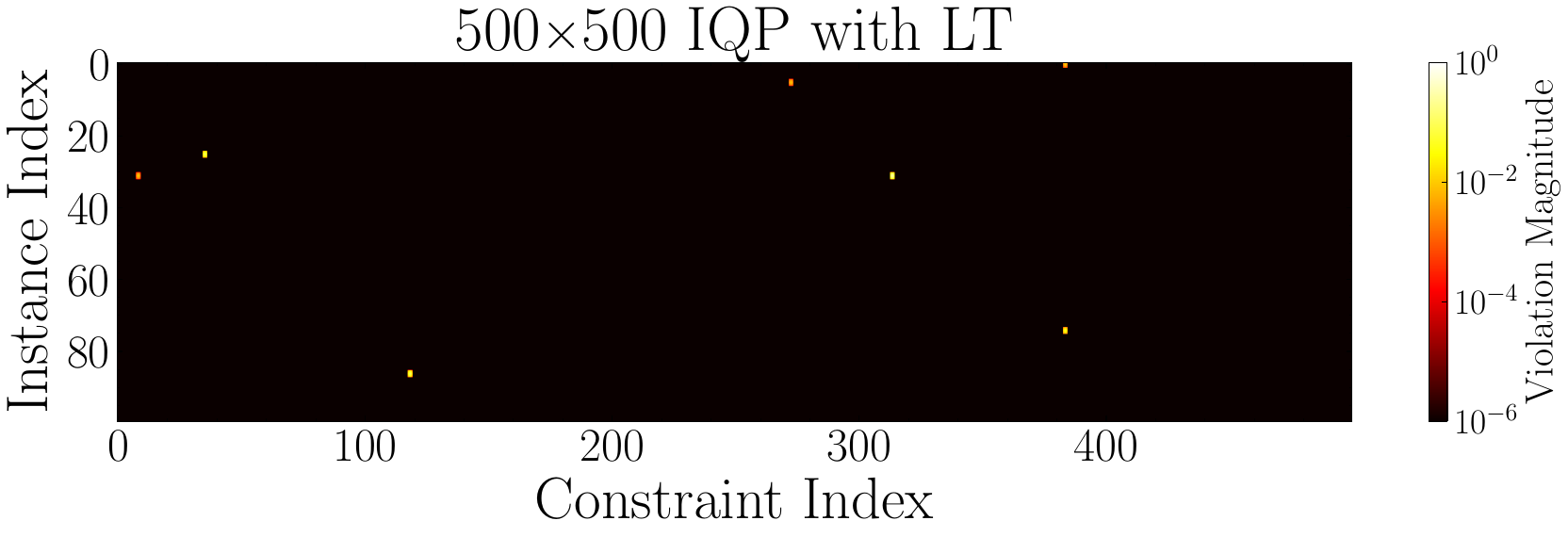}
\end{minipage}

\vspace{-10pt}

\begin{minipage}{0.42\textwidth}
    \centering
    \includegraphics[width=\linewidth]{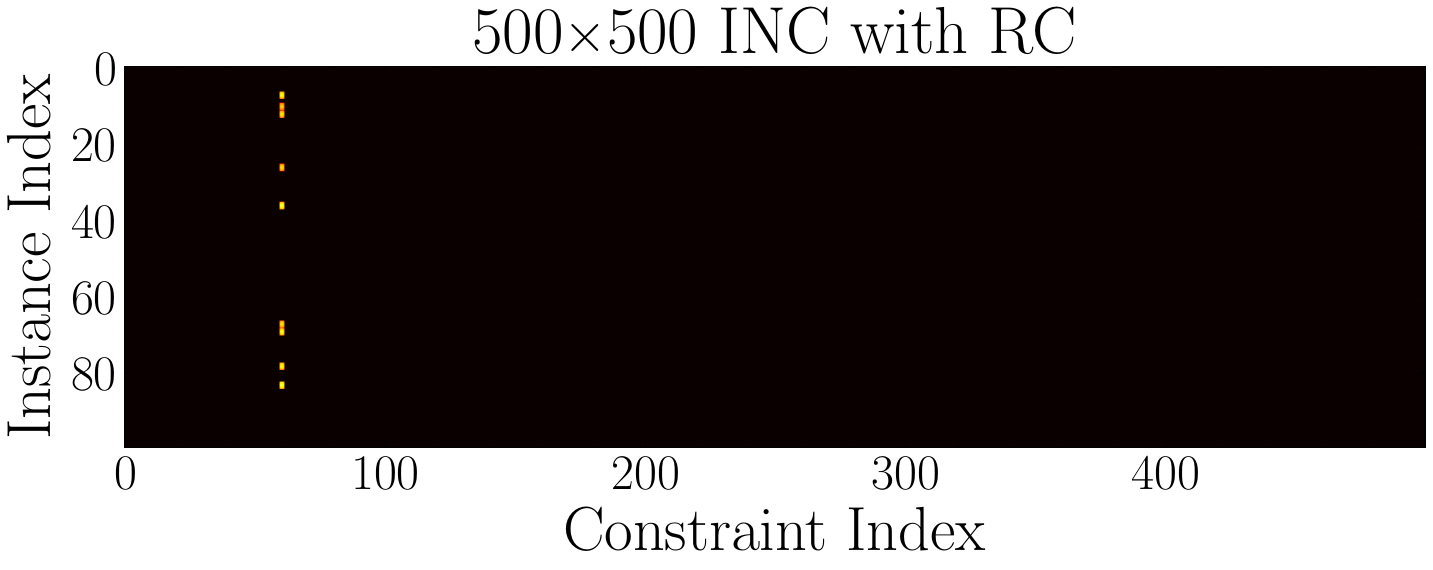}
\end{minipage}
\begin{minipage}{0.49\textwidth}
    \centering
    \includegraphics[width=\linewidth]{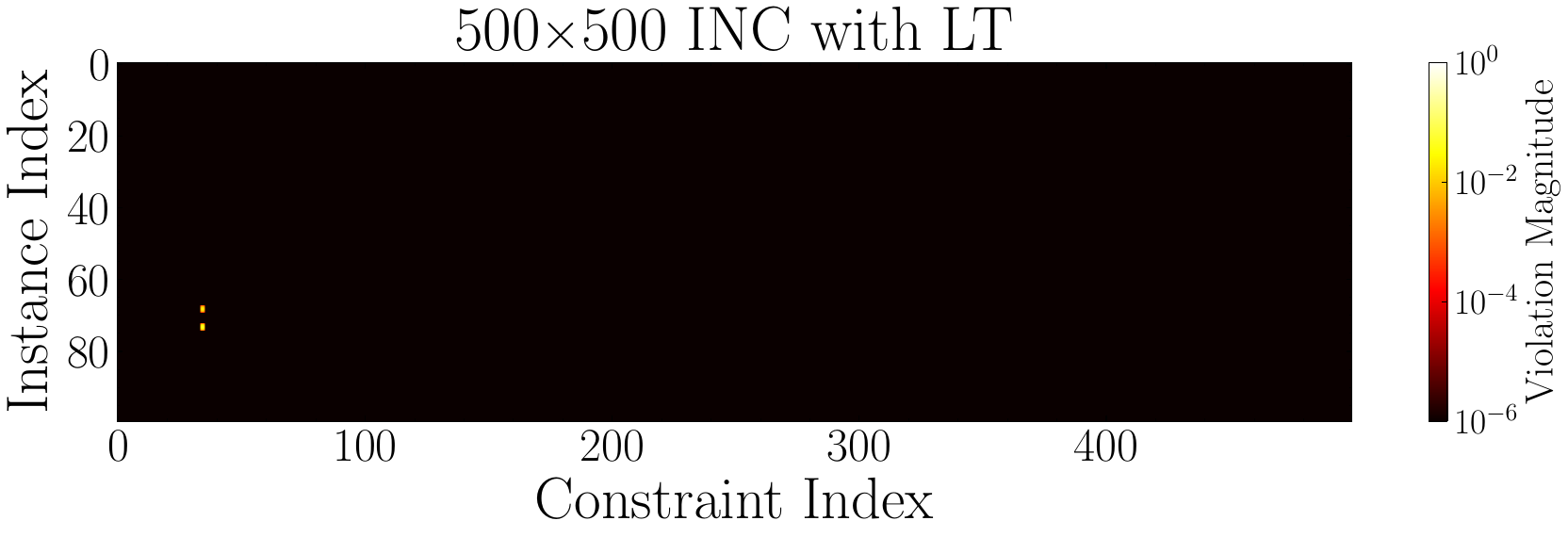}
\end{minipage}

\vspace{-10pt}

\begin{minipage}{0.42\textwidth}
    \centering
    \includegraphics[width=\linewidth]{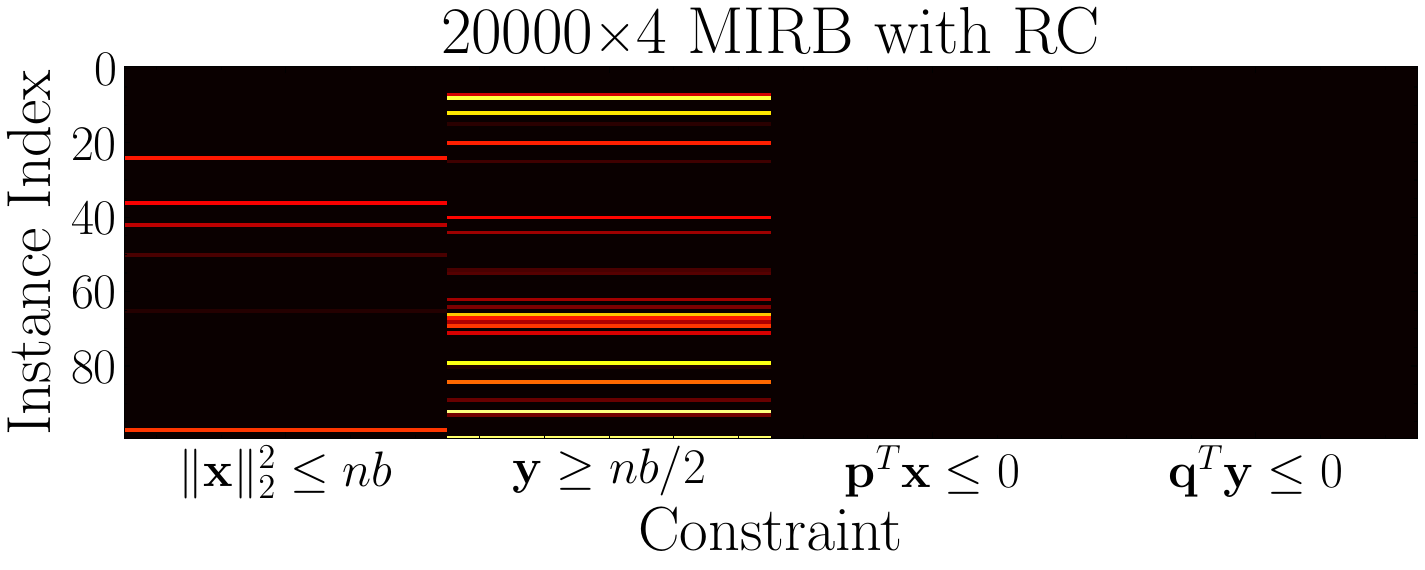}
\end{minipage}
\begin{minipage}{0.49\textwidth}
    \centering
    \includegraphics[width=\linewidth]{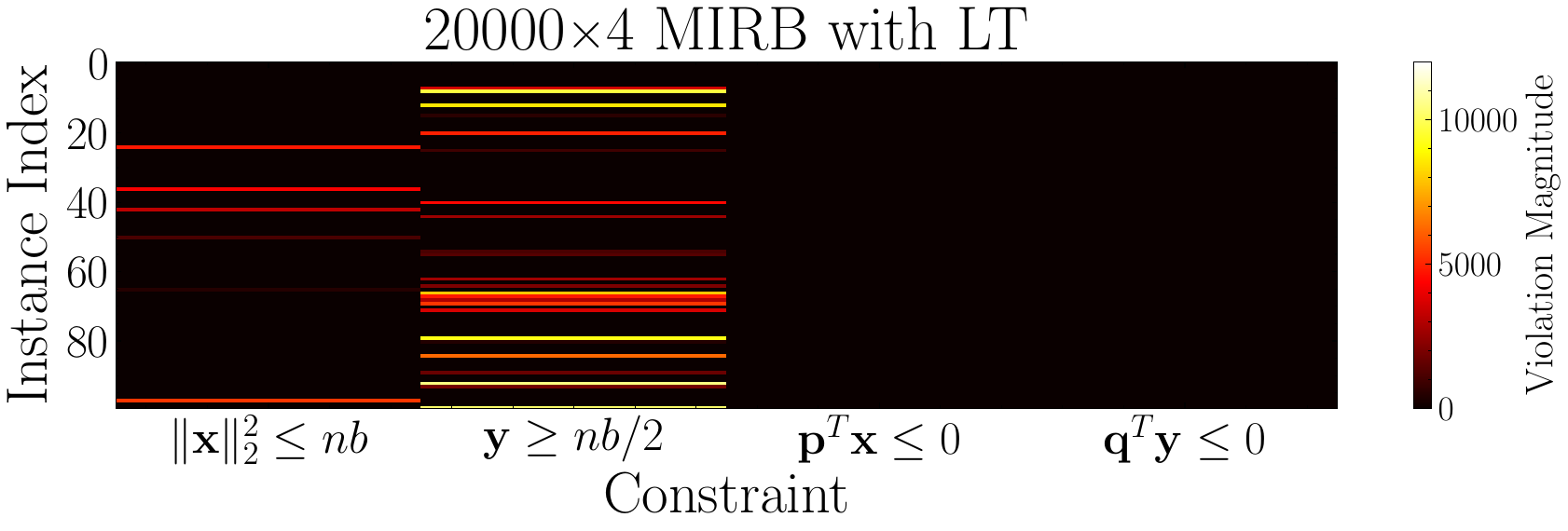}
\end{minipage}

\caption{Constraint violation heatmaps for RC (left column) and LT (right column) across IQPs (top row), INPs (middle row), and MIRBs (bottom row), each evaluated on 100 test instances. Lighter colors indicate greater violation magnitude.}
\label{fig:viol_all}
\vspace{-10pt}
\end{figure*}

\section{STE Rounding Baseline}
\label{app:ste_algo}

\begin{algorithm}[H]
\small
\setlength{\baselineskip}{0.85\baselineskip}
\caption{Rounding with Straight-Through Estimator (Forward Pass)}
\label{algo:ste}
\begin{algorithmic}[1]
\State \textbf{Input:} instance ${\bm{\xi}^i}$; predictor $\pi_{\Theta_1}$
\State Predict relaxed solution $\bar{\bm{x}}^i\gets\pi_{\Theta_1}({\bm{\xi}^i})$
\State Round down: $\hat{{\bm{x}}}^i_z \gets \lfloor \bar{{\bm{x}}}^i_z \rfloor$ 
\State Fractions: ${\bm v}^i \gets \bar{{\bm{x}}}^i_z - \hat{{\bm{x}}}^i_z $
\State Directions: ${\bm b}^i \gets \mathbb{I}({\bm v}^i > 0.5)$
\State Update: $\hat{{\bm{x}}}^i_z \gets \hat{{\bm{x}}}^i_z + {\bm b}^i$
\State \textbf{Output:} $\hat{\bm{x}}^i$
\end{algorithmic}
\end{algorithm}


\end{document}